\documentclass[twoside,11pt]{article}

\usepackage{jmlr2e}

\jmlrheading{18}{2018}{1-38}{10/16; Revised 12/17}{5/18}{16-499}{Aryeh Kontorovich, Sivan Sabato, and Ruth Urner}

\newcommand{\papername}{Active Nearest-Neighbor Learning in Metric Spaces}

\ShortHeadings{\papername}{Kontorovich, Sabato and Urner}

\firstpageno{1}

\title{\papername}

\editor{Andreas Krause}

\author{\name Aryeh Kontorovich  \email karyeh@cs.bgu.ac.il \\
 \name Sivan Sabato \email sabatos@cs.bgu.ac.il \\
\addr Department of Computer Science\\
 Ben-Gurion University of the Negev\\
 Beer Sheva 8499000, Israel\\
\AND
 \name Ruth Urner \email ruth@eecs.yorku.ca\\
 \addr Lassonde School of Engineeging, EECS Department\\
 York University\\
 Toronto, ON, Canada
}

\usepackage{natbib}

\usepackage{times,amsmath,amsfonts,stmaryrd,amssymb,enumitem}
\usepackage{algorithm,booktabs,color}

\usepackage{mathrsfs}
\usepackage{latexsym}
\usepackage{url}

\usepackage{algorithmic}

\usepackage{pifont}
\usepackage{tikz}
\usetikzlibrary{arrows}
\usetikzlibrary{calc}
\usetikzlibrary{shapes}
\usetikzlibrary{fit}
\usetikzlibrary{matrix} 
\usetikzlibrary{positioning}
\usetikzlibrary{decorations.pathreplacing}

\renewcommand{\eqref}[1]{Eq.~(\ref{eq:#1})}

\newcommand{\secref}[1]{Section \ref{sec:#1}}
\newcommand{\thmref}[1]{Theorem \ref{thm:#1}}
\newcommand{\lemref}[1]{Lemma \ref{lem:#1}}

\newcommand{\corref}[1]{Cor.~\ref{cor:#1}}

\newcommand{\algref}[1]{Alg.~\ref{alg:#1}}

\renewcommand{\P}{\mathbb{P}}
\newcommand{\E}{\mathbb{E}}

\newcommand{\hf}{{\textstyle\oo2}}

\newcommand{\reals}{\mathbb{R}}
\newcommand{\nats}{\mathbb{N}}
\newcommand{\half}{{\frac12}}
\newcommand{\ceil}[1]{\ensuremath{\left\lceil#1\right\rceil}}
\newcommand{\floor}[1]{\ensuremath{\left\lfloor#1\right\rfloor}}

\newcommand{\one}{\mathbb{I}}

\DeclareMathOperator*{\argmin}{argmin}
\DeclareMathOperator*{\argmax}{argmax}

\newcommand{\err}{\mathrm{err}}
\newcommand{\st}{\text{ s.t. }}

\newcommand{\beq}{\begin{eqnarray*}}
\newcommand{\eeq}{\end{eqnarray*}}
\newcommand{\beqn}{\begin{eqnarray}}
\newcommand{\eeqn}{\end{eqnarray}}

\newcommand{\hide}[1]{}
\newcommand{\ds}{\displaystyle}

\newcommand{\cA}{\mathcal{A}}

\newcommand{\cD}{\mathcal{D}}

\newcommand{\cF}{\mathcal{F}}

\newcommand{\cH}{\mathcal{H}}

\newcommand{\cL}{\mathcal{L}}

\newcommand{\cN}{\mathcal{N}}

\newcommand{\cT}{\mathcal{T}}

\newcommand{\cX}{\mathcal{X}}
\newcommand{\cY}{\mathcal{Y}}

\newcommand{\ball}{\mathsf{ball}}

\newcommand{\paren}[1]{\left( #1 \right)}
\newcommand{\sqprn}[1]{\left[ #1 \right]}

\newcommand{\set}[1]{\left\{ #1 \right\}}

\newcommand{\diam}{\mathsf{diam}}
\newcommand{\ddim}{\mathsf{ddim}}
\newcommand{\bart}{\bar t}
\newcommand{\setpm}{\set{-1,1}}

\newcommand{\oo}[1]{\frac{1}{#1}}

\newcommand{\bin}{\operatorname{Bin}}

\newcommand{\abs}[1]{\left| #1 \right|}
\newcommand{\nn}{\mathrm{nn}}
\newcommand{\GenNNSet}{\mathsf{GenerateNNSet}}
\newcommand{\ConNet}{\mathsf{Net}}
\newcommand{\netsize}{\cN} 
\newcommand{\Par}{\mathsf{Par}}
\newcommand{\Rec}{\mathsf{Rec}}
\newcommand{\seqN}{N} 
\newcommand{\rek}{k} 
\newcommand{\qb}{Q(m)} 
\newcommand{\proj}[1]{^{#1}} 
\newcommand{\scale}[1]{(#1)} 

\newcommand{\algfont}[1]{\mathsf{#1}}
\newcommand{\genbound}{\mathrm{GB}}
\newcommand{\dist}{\mathrm{Dist}}
\newcommand{\mdist}{\dist_{\mathrm{mon}}}
\newcommand{\uu}{\mathbb{U}}
\newcommand{\sinput}{S_{\mathrm{in}}}
\newcommand{\saug}{S_{\mathrm{a}}}
\newcommand{\spas}{S_{\mathrm{pas}}}
\newcommand{\saughat}{\hat{S}_{\mathrm{a}}}
\newcommand{\uinput}{U_{\mathrm{in}}}
\newcommand{\esterr}{\ensuremath{\algfont{EstimateErr}}}
\newcommand{\selscale}{\ensuremath{\algfont{SelectScale}}}
\newcommand{\hateps}{{\hat{\epsilon}}}
\newcommand{\tmon}{t_m^*}
\newcommand{\ttest}{\cT_{\mathrm{test}}}
\newcommand{\gmin}{G_{\min}}

\newcommand{\algname}{\ensuremath{\algfont{MARMANN}}}
\newcommand{\algnamefull}{MArgin Regularized Metric Active Nearest Neighbor}

\newcommand{\longversion}{} 

\ifdefined\longversion
\newcommand{\longshort}[2]{#1}
\else
\newcommand{\longshort}[2]{#2}
\fi

\begin{document}

\maketitle

\begin{abstract}%
 We propose a pool-based non-parametric active learning algorithm for general metric spaces, called \algnamefull\ (\algname), which outputs a nearest-neighbor classifier. We give prediction error guarantees that depend on the noisy-margin properties of the input sample, and are competitive with those obtained by previously proposed passive learners. We prove that the label complexity of \algname\ is significantly lower than that of any passive learner with similar error guarantees. \algname\ is based on a generalized sample compression scheme, and a new label-efficient active model-selection procedure.
\end{abstract}


\section{Introduction}

Active learning is a framework for
reducing the amount of
label supervision for prediction tasks. 
While labeling large amounts of data can be expensive and time-consuming, unlabeled data is often much easier to come by.
In this paper we propose a non-parametric pool-based active learning algorithm  for general metric spaces, which outputs a nearest-neighbor classifier. 

In pool-based active learning \citep{McCallumNi98}, a collection of random examples is provided, and the algorithm can interactively query an oracle to label some of the examples.
The goal is good prediction accuracy, while keeping the label complexity
---
that is, the number of labels queried
---
low.
Our algorithm, \algnamefull\ (\algname), receives a pool of unlabeled examples in a general metric space, and outputs a variant of the $1$-nearest-neighbor classifier, implemented by a $1$-nearest-neighbor rule.
The algorithm obtains a prediction error guarantee that depends on a noisy-margin property of the input sample, and has a provably smaller label complexity than any passive learner with a similar guarantee.

Active learning has been mostly studied in a parametric setting, in which learning
takes place
with respect to a fixed hypothesis class with a bounded capacity.
There has also been some work on analyzing non-parametric active learning strategies under certain distributional assumptions 	 (see Section \ref{sec:related} for more discussion on this).
However, the question of whether
active querying strategies can yield label savings for non-parametric methods in a general setting, without distributional assumptions,
had not been analyzed prior to this work.
Here, we provide a first demonstration that this is indeed possible.
We discuss related work in detail in \secref{related} below.

\textbf{Our contributions}. \algname\ is a new non-parametric pool-based active
learning algorithm, which obtains an error guarantee competitive with that
of a noisy-margin-based passive learner. Additionally, it 
provably uses significantly fewer labels in nontrivial regimes.
As far as the authors
are aware, this is the first non-parametric active learner for general metric spaces,
which
achieves competitive prediction error guarantees to the passive learner,
while provably improving label complexity. The guarantees of \algname\ are given in \thmref{main} in \secref{main}.
We further provide a passive learning lower bound (\thmref{passivelower}), which together with \thmref{main} shows that \algname\ can have a significantly reduced label complexity compared to any passive learner.
The passive lower bound is
more general than previous lower bounds,
relies on a novel technique,
and may be of independent interest.
Additionally, we give an active label complexity lower bound (\thmref{activelower}), which holds for any active learner with similar error guarantees as \algname. The proof of this active lower bound relies on a new No-Free-Lunch type result, which holds for active learning algorithms.

\textbf{Our approach}. Previous passive learning approaches to classification using nearest-neighbor rules under noisy-margin assumptions \citep{GottliebKN14,gkn-jmlr17+aistats}
provide statistical guarantees using sample compression bounds \citep{DBLP:journals/ml/GraepelHS05}.
Their finite-sample guarantees depend on the number of noisy labels relative to an optimal margin scale.

A central challenge in the active setting is performing model selection to select a margin scale with a low label complexity. 
A key insight that we exploit in this work
is that 
by designing a new labeling scheme for the compression set,
we can construct
the compression set
and
estimate
its error
with label-efficient procedures. 
We obtain statistical guarantees for this approach using generalization bounds for sample compression with side information. 

We derive a label-efficient, as well as computationally efficient, active
model-selection procedure. This procedure finds a good scale by estimating the sample error for some scales, using a small number of 
active querying
rounds. 
Crucially, unlike cross-validation, our model-selection procedure does not
require
a number of labels that depends on the worst possible scale, nor does it test many scales. This allows our label complexity bounds to be low, and to depend only on the
final
scale
selected by the algorithm. Our error guarantee is a constant factor over the error guarantee of the passive learner of \citet{gkn-jmlr17+aistats}. An approach similar to \citet{gkn-jmlr17+aistats}, proposed in \citet{DBLP:journals/tit/GottliebKK14+colt}, has been shown to be Bayes consistent \citep{DBLP:conf/aistats/KontorovichW15}. The Bayes-consistency of the passive version of our approach
has recently been established
\citep{DBLP:conf/nips/KontorovichSW17}.

\textbf{Paper structure}. Related work is discussed in \secref{related}.
We lay down the preliminaries in \secref{setting}.
In \secref{main} we provide our main result: \thmref{main}, which gives error and label complexity guarantees for \algname. Additionally we state the passive and active lower bounds, \thmref{passivelower} and \thmref{activelower}.
The rest of the paper is devoted to the description and analysis of \algname, and proof of the main results. \secref{gennn} shows how \algname\ defines the nearest neighbor rule for a given scale, and \secref{modelselection} describes the model selection procedure of \algname. \thmref{main} is proved in \secref{proofmain}, based on a framework for compression with side information. The passive lower bound in \thmref{passivelower} is proved in \secref{passivelower}. The active lower bound \thmref{activelower} is proved in \secref{activelower}. We conclude with a discussion in \secref{discussion}. 

\subsection{Related Work}\label{sec:related}
The theory of active learning has received
considerable
attention in the past decade \cite[e.g.,][]{Dasgupta04,Balcan2007,BalcanBL09,Hanneke2011, HannekeYa15}.
Active learning theory has been mostly studied in a parametric setting
(that is, learning with respect to a fixed hypothesis class with a bounded capacity).
Benefits and limitations of various active querying strategies have been proven in the realizable setting \citep{Dasgupta04, Balcan2007, GonenSS13+icml} as well as in the agnostic case \citep{BalcanBL09, Hanneke2011, AwasthiBL14}. 
It has also been shown that active queries can also be beneficial for regression tasks \citep{CastroWN05, SabatoM14}.
Further, an active model selection procedure has  been developed for the parametric setting \citep{BalcanHV10}.

The potential benefits of active learning for non-parametric \longshort{settings}{classification in metric spaces} are less well understood.
\longshort{
Practical Bayesian graph-based active learning methods
\citep{Zhu03combiningactive, WeiIB15} rely on generative model assumptions, and therefore come without
distribution-free
performance guarantees.
From a theoretical perspective, the label complexity of graph based active learning has mostly been analyzed in terms of combinatorial graph parameters \citep{Cesa-BianchiGVZ10, DasarathyNZ15}. While the latter work provides some statistical guarantees under specific conditions on the data-generating distribution, this type of analysis does not yield distribution-free statistical performance guarantees.}{}

\citet{CastroWN05, CastroNowak08} analyze minimax rates for non-parametric regression and classification respectively, for a class of distributions in Euclidean space, characterized by decision boundary regularity and noise conditions with uniform marginals. 
The paradigm of cluster-based active learning
\citep{DBLP:conf/icml/DasguptaH08}
has 
been shown to provide label savings under some distributional clusterability assumptions \citep{UrnerWB13, KpotufeUB15}.
\longshort{\citet{DBLP:conf/icml/DasguptaH08}
showed that a suitable
cluster-tree can yield label savings in this framework,
and papers following up \citep{UrnerWB13, KpotufeUB15}
quantified the label savings under
distributional clusterability assumptions.}{} 
However, no active non-parametric strategy has been proposed so far that has label complexity guarantees for i.i.d. data from general distributions and general metric spaces.
Here, we provide the first such algorithm and guarantees.

The passive nearest-neighbor classifier, \longshort{introduced by \citet{FH1951, FH1989},}{}  is popular
among theorists and practitioners alike
\citep{FH1989,CoverHart67, stone1977, KulkarniP95, BoimanSI08}.
This paradigm is applicable in general metric spaces,
and its simplicity is an attractive feature for both
implementation and analysis.
When appropriately regularized
\longshort{
--- either by taking a majority vote among the $k$ nearest neighbors
\citep{stone1977,MR780746\longshort{,MR877849}{}},
or by enforcing a {\em margin} separating the classes
\citep{DBLP:journals/jmlr/LuxburgB04,DBLP:journals/tit/GottliebKK14+colt,DBLP:conf/aistats/KontorovichW15,DBLP:conf/nips/KontorovichSW17} ---}{\cite[e.g.][]{stone1977,MR780746,DBLP:journals/jmlr/LuxburgB04,DBLP:conf/aistats/KontorovichW15}}
this type of learner can be made Bayes-consistent.
Another desirable property of nearest-neighbor-based methods
is their ability to
generalize at a rate that scales with the intrinsic data dimension,
which can be much lower than that of the ambient space
\citep{Kpotufe11,DBLP:journals/tit/GottliebKK14+colt,GottliebKK13tcs+alt,chaudhuri2014}. 
Furthermore,
margin-based regularization makes nearest neighbor classifiers ideally suited for
sample compression, which yields a compact representation, faster
classification runtime, and improved generalization performance
\citep{GottliebKN14}.
The resulting error guarantees
can be stated in terms of the sample's noisy-margin,
which depends on the distances between differently-labeled examples in
the input sample.

Active learning strategies specific to nearest neighbor classification have recently received attention.
It has been shown that certain active querying rules maintain Bayes consistency for nearest neighbor classification, while other, seemingly natural, rules do not lead to a consistent algorithm \citep{Dasgupta12}. 
A selective querying strategy  has
been shown to be beneficial for
nearest neighbors under covariate shift 
\citep{BerlindU15}, where one needs to adapt to a change in the data generating process.
However, the querying rule in that work is based solely on information in the unlabeled data, to account for a shift in the distribution over the covariates.
It does not imply any label savings in the standard learning setting, where training and test distribution are identical.
In contrast, our current work demonstrates how an active learner can take label information into account, to reduce the label complexity of a general nearest neighbor method in the standard setting.

\subsection{A Remark on Bayes-Consistency}\label{sec:consistent}
 We remark on the Bayes-consistency of the margin-based passive $1$-NN
  methods.
  In \citet{DBLP:journals/tit/GottliebKK14+colt},
  a PAC-style generalization bound was given.
  At a given scale $t$, the algorithm first
  ensured $t$-separation of the sample by solving
  solving a minimum vertex cover problem to eliminate the $t$-blocking pairs.
  Following that, the hypothesis was constructed as a
  Lipschitz extension from the remaining
  sample; the latter is computationally implemented as a nearest neighbor classifier.
  Structural Risk Minimization (SRM) was used to select the optimal scale $t$.
  A very close variant of this learner was shown to be Bayes-consistent by
  \citet{DBLP:conf/aistats/KontorovichW15}. The only difference between the two
  is that the former analyzed the hypothesis complexity in terms of fat-shattering dimension while
  the latter via Rademacher averages. Thus, a margin-regularized $1$-NN classifier was shown
  to be Bayes-consistent; however, no compression was involved.

  A compression-based alternative to Lipschitz extension was proposed in
  \citet{GottliebKN14}.
  The idea is again to ensure $t$-separation via vertex cover and then compress the remaining sample
  down to a $t$-net. We conjecture that this latter algorithm is also Bayes-consistent, but currently have no proof.
  If instead one considers a compression-based passive learner implemented as in this paper
  (by taking majority vote
  in each Voronoi region rather than enforcing $t$-separation via vertex cover), the resulting classifier is indeed
  Bayes-consistent, as was recently shown by \cite{DBLP:conf/nips/KontorovichSW17}.

\section{Preliminaries}\label{sec:setting}

In this section we
lay down the
necessary preliminaries. We formally define the setting and necessary notation in \secref{settingsub}. We discuss nets in metric spaces in \secref{tnets}, and present the guarantees of the compression-based passive learner of \cite{gkn-jmlr17+aistats} in \secref{prevalg}.
\subsection{Setting and Notation}\label{sec:settingsub}
For positive integers $n$, denote $[n] := \{1,\ldots,n\}$. 
We consider learning in a general metric space
$(\cX,\rho)$, where $\cX$ is a set and $\rho$ is the metric on $\cX$.
Our problem setting
is that of classification
of the instance space $\cX$
into some finite label set $\cY$.  
Assume that there is some distribution $\cD$ over $\cX \times \cY$, and let $S \sim \cD^m$ be a labeled sample of size $m$, where $m$ is an integer. Denote
the sequence of unlabeled points in $S$ by $\uu(S)$. We sometimes treat $S$ and $\uu(S)$ as multisets, since the order is unimportant. For a labeled multiset $S \subseteq \cX \times \cY$ and $y \in \cY$, denote
$S\proj{y} := \{ x \mid (x,y) \in S\}$;
in particular, $\uu(S) =\cup_{y \in \cY} S\proj{y}$.

The error of a classifier
$h:\cX \rightarrow \cY$ on $\cD$,
for any fixed $h$,
is denoted 
\[\err(h, \cD) := \P[h(X) \neq Y],\]
where $(X,Y) \sim \cD$.
The empirical error on a labeled sample $S$ instantiates to
\[\err(h,S) = \frac{1}{|S|}\sum \one[h(X) \neq Y].\]
A passive learner receives a labeled sample $\sinput$ as input.
An active learner receives the unlabeled part of the sample
$\uinput := \uu(\sinput)$ as input,
and is allowed to
interactively select examples from $\uinput$
and request their label from $\sinput$. In other words, the active learner iteratively selects an example and requests its label, wherein all the labels requested so far can be used to make the next selection.

When
either
learner terminates, it outputs a classifier
$\hat{h}:\cX \rightarrow \cY$, with the goal of achieving a low
$\err(\hat{h},\cD)$.
An additional goal of the active learner is to achieve a performance
competitive with that of the passive learner, while querying considerably fewer labels.

The diameter of a set $A \subseteq \cX$ is defined by
\[
\diam(A):=\sup_{a,a'\in A}\rho(a,a').
\] 
For a finite set $U=\set{u_1,\ldots,u_{|U|}} \subseteq \cX$
with some fixed numbering of its elements,\footnote{
  Invoking the well-ordering principle, we may assume $\cX$ to be well-ordered
  and then any $U\subseteq\cX$ inherits the ordering of $\cX$.
  }
denote the index of the closest point in
$U$ to $x \in \cX$ by 
\[
\kappa(x,U) := \argmin_{i: x_i \in U} \rho(x,x_i).
\]
We assume here and throughout this work that when there
is more than one minimizer for $\rho(x,x_i)$,
ties are broken arbitrarily (but
in a consistent and deterministic fashion).
Any labeled sample $S = ((x_i,y_i))_{i \in [k]}$ naturally
induces the 1-nearest-neighbor classifier
$h^{\nn}_{S}:\cX\to \cY$, via
$h^\nn_{S}(x) := y_{\kappa(x,\uu(S))}$.
For a set $Z \subseteq \cX$, denote by
\[
\kappa(Z,U) := \{ \kappa(z,U) \mid z \in Z\}
\]
the set of all the indices $\kappa(z,U)$, as defined above.
For $x\in\cX$,
and $t>0$,
denote by
$\ball(x,t)$ the (closed) ball of radius $t$ around $x$:
\[
\ball(x,t) := \set{x'\in\cX \mid \rho(x,x')\le t}.
\]

\subsection{Nets}\label{sec:tnets}
A set $A \subseteq \cX$ is \emph{$t$-separated} if $\inf_{a,a' \in A: a\neq a'} \rho(a,a') \geq t$.
For $A\subseteq B\subseteq \cX$,
the set $A$ is a
\emph{$t$-net} of $B$ if
$A$ is $t$-separated
and
$B\subseteq \bigcup_{a\in A}\ball(a,t)$.
Thus, $A$ is a $t$-net of $B$ if it is both a $t$-covering and a $t$-packing.

The size of a $t$-net of a metric space is strongly related to its \emph{doubling dimension}. The doubling dimension is the effective dimension of the metric space,
which controls generalization and runtime performance of nearest-neighbors
\citep{Kpotufe11, DBLP:journals/tit/GottliebKK14+colt}. It is defined as follows.
Let $\lambda = \lambda(\cX)$ be the smallest number such that every
ball in $\cX$ can be covered by $\lambda$ balls of half its 
radius, where all balls are centered at points of $\cX$.
Formally,
\[
\lambda(\cX) := 
\min\{\lambda\in\nats:
\forall x\in\cX,r>0, \quad \exists x_1,\ldots,x_\lambda\in\cX:
\ball(x,r) \subseteq \cup_{i=1}^\lambda \ball(x_i,r/2)
\}
.
\]
Then the doubling dimension of $\cX$ is defined by $\ddim(\cX):=\log_2\lambda$.
In line with modern literature, we work in the low-dimension, large-sample regime, where the doubling dimension is assumed to be constant, and hence sample complexity and algorithmic runtime may depend on it exponentially. This exponential dependence is unavoidable, even under margin assumptions, as
previous analyses \citep{Kpotufe11, DBLP:journals/tit/GottliebKK14+colt}
indicate. 
Generalization bounds in terms of the doubling dimension of the
hypothesis space were established in \citet{Bshouty2009323},
while runtime and generalization errors in terms of $\ddim(\cX)$
were given in \citet{DBLP:journals/tit/GottliebKK14+colt}.

As shown in
\citet{GK-13},
the doubling dimension is ``almost hereditary'' in the sense that
for $A\subset \cX$, we have $\ddim(A)\le c\ddim(\cX)$ for some universal constant $c \leq 2$ 
\citep[Lemma 6.6]{DBLP:journals/corr/FeldmannFKP15}.
In the works cited above, 
where generalization bounds are stated in terms of $\ddim(\cX)$,
one can obtain tighter bounds in terms of $\ddim(\uu(S))$ when the latter is substantially lower
than
that of
the ambient space,
and it is also possible to perform metric dimensionality reduction, as in \citet{GottliebKK13tcs+alt}.

Constructing a minimum size $t$-net for a general set $B$ is NP-hard
\citep{GK-13}.
However, a simple greedy algorithm constructs a (not necessarily minimal) $t$-net in time $O(m^2)$ \citep[Algorithm 1]{GottliebKN14}.
There is also an algorithm for constructing a $t$-net in time $2^{O(\ddim(\cX))}m\log(1/t)$ \citep{KL04,GottliebKN14}.  The size of any $t$-net of a metric space $A \subseteq \cX$ is at most
\begin{equation}\label{eq:ddim-pack}
\ceil{\diam(A)/t}^{\ddim(\cX)+1}
\end{equation}
\citep{KL04}. In addition, the size of any $t$-net is at most $2^{\ddim(A)+1}$ times the size of the minimal $t$-net, as the following easy lemma shows.
\begin{lemma}[comparison of two nets]\label{lem:2net-compare}
  Let $t > 0$ and suppose that $M_1,M_2$ are $t$-nets of $A\subseteq\cX$.
  Then $|M_1|\le 2^{\ddim(A)+1}|M_2|$. 
\end{lemma}
\begin{proof}
  Suppose that $|M_1|\geq k|M_2|$ for some positive integer $k$.
  Since $M_1\subseteq\bigcup_{x\in M_2}\ball(x,t)$,
  it follows from the pigeonhole principle that at least
  one of the points in $M_2$
  must cover at least $k$ points in $M_1$.
  Thus, suppose that $x\in M_2$ covers the set $Z=\set{z_1,\ldots,z_l}\subseteq M_1$,
  meaning that $Z\subseteq\ball(x,t)$, where $l = |Z| \geq k$.
  By virtue of belonging to the $t$-net $M_1$, the set $Z
  $
  is $t$-separated. Therefore $Z$ is a $t$-net of $Z$.
  Since $Z$ is contained in a $t$-ball, we have
  $\diam(Z) \leq 2t$.
  It follows from \eqref{ddim-pack}
  that
  $|Z| \leq 2^{\ddim(A)+1}$,
  whence the claim.
\end{proof}  

Throughout the paper, we
fix a deterministic procedure for constructing a $t$-net, and denote
its output for a multiset $U \subseteq \cX$ by $\ConNet(U,t)$.  
Let $\Par(U,t)$ be a partition of $\cX$ into regions induced by
$\ConNet(U,t)$, that is: for
$\ConNet(U,t) = \{x_1,\ldots,x_\seqN\}$,
define
$\Par(U,t) := \{ P_1,\ldots, P_{\seqN}\}$,
where 
\[
P_i = \{ x \in \cX \mid \kappa(x,\ConNet(U,t)) = i\}.
\]
For $t > 0$, let $\netsize(t) := |\ConNet(\uinput,t)|$ be the size of the $t$-net for the input sample.

\subsection{Passive Compression-Based Nearest-Neighbors}\label{sec:prevalg}
Non-parametric binary
classification admits performance guarantees that scale with the
sample's noisy-margin
\citep{DBLP:journals/jmlr/LuxburgB04,DBLP:journals/tit/GottliebKK14+colt,gkn-jmlr17+aistats}. The original margin-based methods of \citet{DBLP:journals/jmlr/LuxburgB04}
and \citet{DBLP:journals/tit/GottliebKK14+colt}
analyzed the generalization performance via the technique of Lipschitz
  extension. Later, it was noticed in
  \citet{GottliebKN14} that the presence of a margin allows for compression --- in fact,
  nearly optimally so.

We say that a labeled multiset $S$ is \emph{$(\nu,t)$-separated},
for $\nu \in [0,1]$ and $t > 0$ (representing a margin $t$ with noise $\nu$), if one can remove
a $\nu$-fraction of the points in $S$, and
in the resulting multiset,
any pair of differently labeled points is separated by a distance of at least $t$.
Formally, we have the following definition.
\begin{definition}
$S$ is $(\nu,t)$-separated if there exists a subsample $\tilde{S} \subseteq S$ such that
\begin{enumerate}
\item $|S \setminus \tilde{S}|
\leq \nu|S|$ and
\item $\forall y_1 \neq y_2 \in \cY, a\in\tilde{S}\proj{y_1},b\in\tilde{S}\proj{y_2}$, we have $\rho(a,b) \geq t$. 
\end{enumerate}
\end{definition}
For a given labeled sample $S$, denote by $\nu(t)$ the smallest value
$\nu$ such that $S$ is $(\nu,t)$-separated. \cite{gkn-jmlr17+aistats} propose
a passive learner with the following guarantees\footnote{The guarantees hold for the more general case of semimetrics.} as a function of the separation of $S$. 
Setting $\alpha := m/(m-\seqN)$, define the following form of a generalization bound:
\begin{align*}
\genbound(\epsilon,\seqN,\delta,m,\rek) &:=  \alpha  \epsilon  +
  \frac{2}{3}\frac{(\seqN+1)\log(m\rek)+\log(\oo\delta)}{m-\seqN}
  +\frac{3}{\sqrt{2}}\sqrt{ \frac{\alpha  \epsilon ((\seqN+1)
      \log(m\rek)+\log(\oo\delta))}{m-\seqN}}.
\end{align*}
Further, for an integer $m$ and $\delta \in (0,1)$, denote
\[
\gmin(m,\delta) := \min_{t > 0} \genbound(\nu(t),\netsize(t),\delta, m, 1).
\]
The quantity $\gmin(m,\delta)$ is small for datasets where we only need to remove few points to obtain a reasonably well separated subset. 
As an example, consider data generated by two well separated Gaussians (one generating the positively labeled points, and one generating the negatively labeled points). 
Then, most of the data points will be close to their respective means, but some will be farther, and may lie closer to the mean of the other Gaussian. Removing those few will result in a separated set.
\begin{theorem}[\citealt{gkn-jmlr17+aistats}]
  \label{thm:passive-UB}
  Let $m$ be an integer, $\delta \in (0,1)$.
  There exists a passive learning algorithm that returns a nearest-neighbor classifier $h^\nn_{\spas}$, where $\spas \subseteq \sinput$, such that, with probability $1-\delta$,
  \begin{align*}
    \err(h^\nn_{\spas},\cD) \leq
\gmin(m,\delta).
  \end{align*}
\end{theorem}
The bound above is data-dependent, meaning that the strength of the generalization guarantee
depends on the quality of the random sample.
Specifically,
the passive algorithm of \citet{gkn-jmlr17+aistats} generates $\spas$ of size approximately $\netsize(t)$ for the optimal scale $t>0$ (found by searching over all scales),
by removing the $|\sinput| \nu(t)$ points that
obstruct
the $t$-separation
between different labels in $\sinput$, and then selecting a subset of the
remaining labeled examples to form $\spas$,
so that the examples are a $t$-net for $\sinput$ (not including the obstructing points).
For the binary classification case ($|\cY| = 2$) an efficient algorithm is shown in \citet{gkn-jmlr17+aistats}. However, in the general multiclass case, it is not known how to find a minimal $t$-separation efficiently --- a naive approach requires solving the NP-hard problem of vertex cover. Our approach, which we detail below, circumvents this issue, and provides an efficient algorithm also for the multiclass case.

\section{Main Results}\label{sec:main}

We propose a novel approach for generating a subset for a  nearest-neighbor rule. This approach, detailed in the following sections, does not require finding and removing all the
obstructing
points in $\sinput$, and can be implemented in an active setting using a small number of labels. The resulting active learning algorithm, \algname, has an error guarantee
competitive with that of the passive learner,
and a label complexity that can be significantly lower. We term the subset used by the nearest-neighbor rule a \emph{compression set}.

\begin{algorithm}
\caption{$\algname$: \algnamefull}
\label{alg:main}
\begin{algorithmic}
\REQUIRE Unlabeled sample $\uinput$ of size $m$, $\delta \in (0,1)$. 
\STATE $\hat{t} \leftarrow \selscale(\delta)$. $\qquad\qquad\qquad\qquad$ \# $\selscale$ is given in \secref{modelselection}, \algref{tSearch}.
\STATE $\hat{S} \leftarrow \GenNNSet(\hat{t},[\netsize(\hat{t})],\delta)$. $\qquad$ \# $\GenNNSet$ is given in \secref{gennn}, \algref{genclass}.
\STATE Output $h^{\nn}_{\hat{S}}$. 
\end{algorithmic}
\end{algorithm}

\algname, listed in \algref{main}, operates as follows. First, a scale $\hat{t} > 0$ is selected, by calling $\hat{t} \leftarrow \selscale(\delta)$, where \selscale\ is our model selection procedure. \selscale\ has access to $\uinput$, and queries labels from $\sinput$ as necessary. It estimates the generalization error bound $\genbound$ for several different scales, and
executes
a procedure similar to binary search to identify a good scale. The binary search keeps the number of estimations (and thus requested labels) small. Crucially, our estimation procedure is designed to prevent the search from spending a number of labels that depends on the net size of the smallest possible scale $t$, so that the total label complexity of \algname\
depends only on the error of the selected $\hat{t}$. 
Second, the selected scale $\hat{t}$ is used to generate the compression set by
calling $\hat{S} \leftarrow \GenNNSet(\hat{t},[\netsize(\hat{t})],\delta)$,
where $\GenNNSet$ is our procedure for generating the compression set, and $[\netsize(\hat{t})] \equiv \{1,\ldots,\netsize(\hat{t})\}$. 
Our main result is the following guarantee for \algname.

\begin{theorem}[Main result; Guarantee for \algname] \label{thm:main}
Let $\sinput \sim \cD^m$, where $m \geq \max(6,|\cY|)$, $\delta \in (0,\oo4)$. 
Let $h^{\nn}_{\hat{S}}$ be the output of $\algname(\uinput, \delta)$, 
where $\hat{S} \subseteq \cX \times \cY$,
and let $\hat{\seqN} := |\hat{S}|$.
Let $\hat{h} := h^\nn_{\hat{S}}$
and
$\hat{\epsilon} := \err(\hat{h},\sinput)$,
and denote $\hat{G} := \genbound(\hateps,\hat{\seqN},\delta,m,1)$.
With a probability of $1-\delta$ over $\sinput$ and randomness of \algname,
\[
\err(\hat{h},\cD) \leq 2\hat{G} \leq O\left(\gmin(m,\delta)\right),
\]
and the number of labels from $\sinput$ requested by \algname\ is at most
\begin{equation}
  \label{eq:act-sampl-cmpl}
O\left(\log^3(\frac{m}{\delta})\left(\frac{1}{\hat{G}}\log(\frac{1}{\hat{G}}) + m \hat{G}\right)\right).
\end{equation}
Here the $O(\cdot)$ notation hides only universal multiplicative constants.
\end{theorem}
Our error guarantee is thus a constant factor over the error guarantee of the passive learner of \citep{gkn-jmlr17+aistats}, given in \thmref{passive-UB}.
The constant factor that we derive in our analysis is in the order of $2000$ (this can be seen in the proof of \thmref{selscalegen}). Note that we did not focus on optimizing it, opting instead for a more streamlined analysis. 
As the lower bound in \thmref{activelower} shows,
the additive term $m \hat{G}$ in \eqref{act-sampl-cmpl} is essentially unavoidable.
Whether the dependence on $1/\hat G$ is indeed necessary is currently an open problem.

To observe the advantages of \algname\
over a passive learner,
consider a scenario in which the upper bound
$\gmin$
of
\thmref{passive-UB},
as well as the Bayes error of $\cD$, are of order $\tilde{\Theta}(1/\sqrt{m})$.
Then $\hat{G} = \Theta(1/\sqrt{m})$ as well. Therefore,
\algname\ obtains a prediction error guarantee of $\tilde\Theta(1/\sqrt{m})$,
similarly to the passive learner, but it uses only $\tilde{\Theta}(\sqrt{m})$
labels instead of $m$. 
In contrast, the following result shows that no learner that selects labels 
uniformly at random from $\sinput$ can compete with
\algname: \thmref{passivelower} below
shows that
for any passive learner that uses $\tilde{\Theta}(\sqrt{m})$
random labels from $\sinput$, there exists a distribution $\cD$ with
the above properties, for which the prediction error of the passive learner in this case is
${\tilde\Omega}(m^{-1/4})$, a decay rate which is
almost
quadratically slower than the 
$\tilde{O}(1/\sqrt{m})$ rate achieved by \algname.
Thus, the guarantees of \algname\ cannot be matched by any passive learner.

\begin{theorem}[Passive lower bound]
  \label{thm:passivelower}
  Let $m > 0$ be an integer, and suppose that $(\cX,\rho)$ is a metric space such that for some
  $\bart > 0$,
  there is a $\bart$-net $T$ of $\cX$
of size
$\Theta(\sqrt m)$.
Consider any
passive learning
algorithm that maps i.i.d.~samples $S_\ell\sim\cD^\ell$ from some distribution $\cD$ over $\cX\times\setpm$, to functions $\hat{h}_\ell:\cX \rightarrow \setpm$. 
For any such algorithm and any $\ell = \tilde{\Theta}(\sqrt{m})$,
there exists a distribution $\cD$ such that:
\begin{enumerate}
\item[i.] The Bayes error of $\cD$ is $\Theta(1/\sqrt{m})$;
\item[ii.]
  With at least a constant probability, both of the following events occur:
  \begin{enumerate}
  \item[(a)]
    The passive learner achieves error
    $\err(\hat{h}_\ell,\cD) = \tilde{\Omega}(m^{-1/4})$,
  \item[(b)] ${\ds
    \gmin(m,\delta)  = \tilde\Theta(1/\sqrt{m}).
    }$
  \end{enumerate}
\end{enumerate}
Furthermore,
i. and ii. continue to hold
when the learning
algorithm has access to the full
marginal distribution over $\cX$.
\end{theorem}
Thus,
\algname\ even improves over a semi-supervised learner:
its label savings
stem from actively selecting labels,
and are not achievable by merely exploiting information from unlabeled data
or by randomly selecting examples to label.

We deduce \thmref{passivelower} from a more general result, which might be of independent interest. \thmref{passivelower-gen}, given in \secref{passivelower}, improves existing passive learning sample complexity lower bounds.
In particular, our result removes
the restrictions of previous lower bounds
on the relationship between the sample size, the VC-dimension, and the noise level,
which render existing bounds
inapplicable
to our parameter regime.
The proof of \thmref{passivelower}  is given thereafter in \secref{passivelower}, as a consequence of \thmref{passivelower-gen}.

We further provide a label complexity lower bound, in \thmref{activelower} below, which holds for any active learner that obtains similar guarantees to those of \algname. The lower bound shows that any active learning algorithm which guarantees a multiplicative accuracy over $\gmin(m,\delta)$ has a label complexity which is $\tilde{\Omega}(m\gmin(m,\delta))$, for a wide range of values of $\gmin(m,\delta)$ --- essentially, as long as $\gmin(m,\delta)$ is not trivially large or trivially small. This implies that the term $m\hat{G}$ in the upper bound of the label complexity of \algname\ in \thmref{main} cannot be significantly improved.

\begin{theorem}[Active lower bound]\label{thm:activelower}
Let $\cX = \reals$, $\delta \in (0,1/14)$.
Let $C \geq 1$, and let $\cA$ be an active learning algorithm that outputs $\hat{h}$. 
Suppose that for any distribution $\cD$ over $\cX \times \cY$, if the input unlabeled sample is of size $m$, then 
$\err(\hat{h},\cD) \leq C\gmin(m,\delta)$ with probability at least $1-\delta$. 
Then for any $\alpha \in(\frac{\log(m) + \log(28)}{8\sqrt{2m}},\frac{1}{240C})$  there exists an a distribution $\cD$ such with probability at least $\frac{1}{28}$ over $S \sim \cD^m$ and the randomness of $\cA$, both of the following events hold:
\begin{enumerate}
\item $\alpha \leq \gmin(m,\delta) \leq 30 \alpha$
\item $\cA$ queries at least $\half\floor{\frac{m \gmin(m,\delta) - \log(\frac{m}{\delta})}{30\log(m)}} \equiv \tilde{\Omega}(m \gmin(m,\delta))$ labels.
\end{enumerate}
\end{theorem}

The proof of this lower bound is provided in \secref{activelower}.
In the rest of the paper, the components of \algname\ are described in detail, and the main results are proved.

\section{Active Nearest-Neighbor at a Given Scale}\label{sec:gennn}

A main challenge for active learning in our non-parametric setting is performing model selection, that is, selecting a good scale $t$ similarly to the passive learner of \cite{gkn-jmlr17+aistats}. 
In the passive supervised setting, the approach developed  in several previous works 
\citep{GottliebKN14,DBLP:conf/icml/KontorovichW14,DBLP:journals/tit/GottliebKK14+colt,DBLP:conf/aistats/KontorovichW15} performs model selection by solving a minimum vertex cover problem for each considered scale $t$,
so as to eliminate all of the {\em $t$-blocking pairs} --- i.e., pairs of differently labeled points within a distance $t$. The passive algorithm generates a compression set by first finding and removing from $\sinput$ all
points
that obstruct $(\nu,t)$-separation
at a given
scale $t>0$. 
This incurs a computational cost but no significant sample complexity
increase,
aside from the standard logarithmic factor that comes from stratifying
over data-dependent hierarchies \citep{DBLP:journals/tit/Shawe-TaylorBWA98}.

While this approach works for passive learning, in the active setting we face a crucial challenge: estimating the error of a nearest-neighbor rule at scale $t$ using a small number of samples. A key insight that we exploit in this work 
is that instead of
eliminating the blocking pairs, one may simply relabel some of the points in the compression set,
 and this would also generate a low-error nearest neighbor rule.
This new approach enables estimation of the sample accuracy of a (possibly
relabeled) $t$-net by label-efficient active sampling. In addition, this approach is significantly
simpler than estimating the size of the minimum vertex cover of the
$t$-blocking graph. 
Moreover, we gain improved algorithmic efficiency, by avoiding the relatively expensive vertex cover procedure.

A small technical difference, which will be evident below, is that in this new approach, examples in the compression set might have a different label than their original label in $\sinput$.  Standard sample
compression analysis \citep[e.g.][]{DBLP:journals/ml/GraepelHS05} assumes that the classifier is determined by a small number of labeled examples from $\sinput$. This does not allow the examples in the compression set to have a different label than their original label in $\sinput$. Therefore, we require a slight generalization of previous compression analysis (following previous works on compression, see details in \secref{compress}), which allows adding side information to the compression set. This side information will be used to set the label of each of the examples in the compression set. The generalization incurs a small statistical penalty, which we quantify in \secref{proofmain}, as a preliminary to proving \thmref{main}.

We now
describe our approach to generating a compression set for a given scale $t>0$. Recall that $\nu(t)$ is the smallest value
for which
$\sinput$ is $(\nu,t)$-separated.
We define two compression sets. The first one, denoted $\saug\scale{t}$, represents an
ideal compression set, constructed
(solely for the sake of analysis)
so that it induces an empirical error of at most $\nu(t)$. Calculating $\saug\scale{t}$
might require many labels, thus it is only used for analysis purposes;
the algorithm never constructs it.
The second compression set, denoted $\saughat\scale{t}$,
represents an approximation to $\saug\scale{t}$,
which can be constructed using a small number of labels,
and induces a sample error of at most $4\nu(t)$ with high probability.

We first define the ideal set $\saug\scale{t} := \{(x_1,y_1),\ldots,(x_\seqN,y_\seqN)\}$. The examples in $\saug\scale{t}$ are the points in $\ConNet(\uinput,t/2)$,
and the label of each example is the majority label, out of the labels of the examples in $\sinput$ to which $x_i$ is closest. Formally, $\{x_1,\ldots,x_\seqN\} :=
\ConNet(\uinput,t/2)$, and for $i \in [\seqN]$,
$y_i := \argmax_{ y\in \cY} |S\proj{y} \cap P_i|$,
where $P_i = \{ x \in \cX \mid \kappa(x,\ConNet(U,t/2)) = i\} \in \Par(\uinput,t/2)$. 

\begin{lemma}\label{lem:multiupper}
  Let $S$ be a labeled sample of size $m$, and let $\{P_1,\ldots,P_\seqN\}$ be a partition of $\uu(S)$, with $\max_i\diam(P_i) \leq t$ for some $t \geq 0$. For $i\in[\seqN]$, let $\Lambda_i:=S\proj{y_i} \cap P_i$. Then 
\[
\nu(t) \geq 1-\frac{1}{m}\sum_{i \in [\seqN]} |\Lambda_i|. 
\]
\end{lemma}
\begin{proof}
  Let $\tilde{S} \subseteq S$ be a subsample that
  witnesses the $(\nu(t),t)$-separation of $S$,
  so that \mbox{$|\tilde{S}| \geq m(1- \nu(t))$}, and for any two points
  $(x,y),(x',y') \in \tilde{S}$,
  if $\rho(x,x') \leq t$ then $y = y'$.
  Denote $\tilde{U} := \uu(\tilde{S})$.
  Since $\max_{i}\diam(P_i) \leq t$, for any $i \in [\seqN]$ all the points in $\tilde{U} \cap P_i$ must have the same label in $\tilde{S}$.
  Therefore, 
\[
\exists y \in \cY \st \tilde{U} \cap P_i \subseteq   \tilde{S}\proj{y} \cap P_i.
\]
Hence
  $|\tilde{U} \cap P_i| \leq |\Lambda_i|$. It follows
\[
|S|-\sum_{i \in [\seqN]} |\Lambda_i| \leq |S| - \sum_{i \in [\seqN]}|\tilde{U}
\cap P_i| = |S| - |\tilde{S}| = m \cdot \nu(t).
\]
Dividing by $m$ we get the statement of the lemma.
\end{proof}

From \lemref{multiupper}, we get \corref{errleqnu}, which upper bounds the empirical error of $h^\nn_{\saug\scale{t}}$ by $\nu(t)$.
\begin{corollary}\label{cor:errleqnu}
For every $t > 0$, $\err(h^\nn_{\saug\scale{t}},\sinput) \leq \nu(t)$.
\end{corollary}
This corollary is immediate from \lemref{multiupper}, since for any $P_i \in \Par(\uinput,t/2)$, $\diam(P_i) \leq t$, and 
\[
\err(h^\nn_{\saug\scale{t}},\sinput) = 1-\frac{1}{m}\sum_{i \in [\seqN]} |\Lambda_i|.
\]

Now, calculating $\saug\scale{t}$ requires knowing most of the labels in $\sinput$. \algname\ constructs instead an
approximation
$\saughat\scale{t}$, in which the examples are the points in $\ConNet(\uinput,t/2)$ (so that $\uu(\saughat\scale{t}) = \uu(\saug\scale{t})\,$), but the labels are 
determined
using a bounded number of labels requested  from $\sinput$. 
The labels in $\saughat\scale{t}$ are calculated by the simple procedure $\GenNNSet$ given in \algref{genclass}. The empirical error of the output of $\GenNNSet$ is bounded in \thmref{nnset} below.\footnote{In the case of binary labels
($|\cY| = 2$), the problem of estimating $\saug\scale{t}$ can be formulated as a special case of the benign noise setting for parametric active learning, for which tight lower and upper bounds are provided in \cite{HannekeYa15}. However, our case is both more general (as we allow multiclass labels) and more specific (as we are dealing with a specific ``hypothesis class''). Thus we provide our own procedure and analysis.}

A technicality in \algref{genclass} requires explanation: In \algname, the generation of $\saughat\scale{t}$ will be split into several calls to $\GenNNSet$, so that different
calls determine the labels of different 
points
in $\saughat\scale{t}$. Therefore $\GenNNSet$ has an additional argument $I$, which
specifies the indices of the points in $\ConNet(\uinput, t/2)$ for which the labels should be returned this time.
Crucially, if during the run of \algname, $\GenNNSet$ is called again for the same scale $t$ and the same point in $\ConNet(\uinput, t/2)$, then $\GenNNSet$ returns the same label that it returned before, rather than recalculating it using fresh labels from $\sinput$. This guarantees
that despite the randomness in $\GenNNSet$, the full $\saughat\scale{t}$ is well-defined within any single run of \algname, and is distributed like the output of $\GenNNSet(t,[\netsize(t/2)],\delta)$, which is convenient for the analysis. 
Define
\begin{equation}\label{eq:qb}
\qb := \ceil{18\log(4m^3/\delta)}.
\end{equation}

\begin{algorithm}
\caption{$\GenNNSet(t, I, \delta)$ } 
\label{alg:genclass}
\begin{algorithmic}
\REQUIRE Scale $t > 0$, a target set $I \subseteq [\netsize(t/2)]$, confidence $\delta \in (0,1)$.
\ENSURE A labeled set $S \subseteq \cX \times \cY$ of size $|I|$
\STATE $\{x_1,\ldots,x_\seqN\} \leftarrow \ConNet(\uinput,t/2)$, $\{P_1,\ldots, P_\seqN\} \leftarrow \Par(\uinput,t/2)$, $S \leftarrow ()$ %
\FOR{$i \in I$}
\IF{$\hat{y}_i$ has not already been calculated for $\uinput$ with this
  value of $t$}
\STATE Draw $\qb$ points uniformly at random from $P_i$ and query their labels. 
\STATE Let $\hat{y}_i$ be the majority label observed in these $\qb$ queries.
\ENDIF
\STATE $S \leftarrow S \cup  \{(x_i, \hat{y}_i)\}$. 
\ENDFOR
\STATE Output $S$
\end{algorithmic}
\end{algorithm}

\begin{theorem}\label{thm:nnset}
Let $\saughat\scale{t}$ be the output of $\GenNNSet(t,[\netsize(t/2)],\delta)$.
With a probability at least 
$1-\frac{\delta}{2m^2}$, the following event, which we denote by $E(t)$, holds:
\[
\err(h^\nn_{\saughat\scale{t}}, \sinput) \leq 4 \nu(t). 
\]
\end{theorem}
\begin{proof}
  By \corref{errleqnu}, $\err(h^\nn_{\saug\scale{t}}, \sinput) \leq \nu(t)$.
  In $\saug\scale{t}$, the labels assigned to each point in $\ConNet(\uinput, t/2)$ are
  the majority labels (based on $\sinput$) of the points in the regions in $\Par(\uinput,t/2)$.
  As above, we denote the majority label for region $P_i$ by $y_i := \argmax_{y \in \cY}
  |S\proj{y} \cap P_i|$. 
  We now compare these labels to the labels $\hat{y}_i$ assigned by \algref{genclass}.
  Let $p(i) = |\Lambda_i|/|P_i|$ be the fraction of points in $P_i$
  which are labeled by the majority label $y_i$, where $\Lambda_i$ is as defined in \lemref{multiupper}.
  Let $\hat{p}(i)$ be the fraction of labels equal to $y_i$ out of those queried by \algref{genclass} in round $i$. 
  Let $\beta:=1/6$. By Hoeffding's inequality and union bounds,
  we have that with a probability of at least 
\[
1-2\netsize(t/2)\exp(-\frac{\qb}{18}) \geq 1 - \frac{\delta}{2m^2},
\]
we have 
  $\max_{i\in[
      \netsize\scale{t/2}
  ]} \abs{\hat p(i)-p(i)} \le \beta$. Denote this ``good'' event by $E'$. We now prove that $E' \Rightarrow E(t)$. 
Let $
J = \{ i\in
[\netsize\scale{t/2}]
\mid \hat p(i)>\half\}$.
  It can be easily seen that
  $\hat{y}_i = y_i$
  for all $i \in J$.
  Therefore, for all $x$ such that $\kappa(x,\uu(\saug\scale{t})) \in J$, $h^\nn_{\saughat\scale{t}}(x) = h^{\nn}_{\saug\scale{t}}(x)$, and hence
\[
\err(h^\nn_{S},\sinput) \leq \P_{X \sim \sinput}[\kappa(X,\uu(\saug\scale{t})) \notin J] + \err(h^\nn_{\saug\scale{t}},\uinput).
\]
  The second term is at most $\nu(t)$ by \corref{errleqnu}, and it remains to bound the first term,
on the condition
that $E'$ holds.
We have 
$\P_{X \sim U}[\kappa(X,\uu(\saug\scale{t})) \notin J] =\frac{1}{m}\sum_{i \notin J}|P_i|$. 
 If $E'$ holds, then for any $i \notin J$, $p(i) \leq \half + \beta$, therefore 
\[
|P_i|-|\Lambda_i| = (1-p(i))|P_i| \geq  (\half - \beta)|P_i|.
\]
Recall that, by \lemref{multiupper},
$ \nu(t) \geq 1-\frac{1}{m}\sum_{i \in [\netsize\scale{t/2}]}|\Lambda_i|$. 
Therefore,
\begin{align*}
 \nu(t) &\geq  1-\frac{1}{m}\sum_{i \in [\netsize\scale{t/2}]}|\Lambda_i|\\
&=  \frac{1}{m}\sum_{i \in [\netsize\scale{t/2}]}(|P_i| - |\Lambda_i|) \\
&\geq  \frac{1}{m}\sum_{i \notin J}(|P_i| - |\Lambda_i|)\\
&\geq \frac{1}{m}\sum_{i \notin J}(\half - \beta)|P_i|.
\end{align*}
Thus, under $E'$, 
\[
\P_{X \sim U}[\kappa(X,\uu(\saug\scale{t})) \notin J] \leq \frac{\nu(t)}{\hf - \beta} = 3\nu(t).
\]
It follows that under $E'$, $\err(h^\nn_{S},\uinput) \leq 4\nu(t)$.
\end{proof}

\section{Model Selection}\label{sec:modelselection}

We now show how to select the scale $\hat{t}$ that will be used to generate the output nearest-neighbor rule. The main challenge is to do this with a low label complexity: 
Generating the full classification rule for scale $t$ requires a number of labels that depends on $\netsize(t)$, which might be very large. We would like the label complexity of \algname\ to depend only on $\netsize(\hat{t})$ (where $\hat{t}$ is the selected scale), which is of the order $m \hat{G}$. Therefore, during model selection we can only invest a bounded number of labels in each tested scale. In addition, to keep the label complexity low, we would like to avoid testing all scales.
In \secref{estimateerr} we describe how we estimate the error on a given scale. In \secref{binary} we provide a search procedure, resembling binary search, which uses the estimation procedure to select a single scale $\hat{t}$. 

\subsection{Estimating the Error at a Given Scale}\label{sec:estimateerr}

For $t > 0$, let $\saughat\scale{t}$ be the compressed sample that \algname\ would
generate if the selected scale were set to $t$.  Our model selection procedure
performs a search, similar to binary search, over the possible scales. For each
tested scale $t$, the procedure estimates the empirical error
$\epsilon\scale{t} := \err(h^\nn_{\saughat\scale{t}},S)$ within a certain
accuracy, using an estimation procedure given below, called \esterr. \esterr\ outputs an estimate $\hat{\epsilon}\scale{t}$ of $\epsilon\scale{t}$, up to a given threshold $\theta > 0$, using labels requested from $\sinput$.  

\newcommand{\selalg}{

\begin{algorithm}[t]
\caption{$\selscale(\delta)$}
\label{alg:tSearch}
\begin{algorithmic}[1]
\REQUIRE $\delta \in (0,1)$
\ENSURE Scale $\hat{t}$
\STATE $\cT \leftarrow \mdist$,  \hspace{2em}\# $\cT$ maintains the current set of possible scales
\WHILE{$\cT \neq \emptyset$}
\STATE{$t \leftarrow$ the median value in $\cT$ \hspace{3em} \#  break ties arbitrarily}
\STATE{$\hat{\epsilon}(t) \leftarrow \esterr(t, \phi(t), \delta).$ }

 \IF{$\hateps(t) < \phi(t)$} 
 \STATE{$\cT \leftarrow \cT \setminus [0,t]$ \# go right in the binary search \label{step:right}}
 \ELSIF{$\hateps(t) > \frac{11}{10} \phi(t)$} 
 \STATE{$\cT \leftarrow \cT \setminus [t, \infty)$ \# go left in the binary search \label{step:left}}
 \ELSE \STATE{$t_0 \leftarrow t, \cT_0 \leftarrow \{t_0\}$.}
\STATE{\textbf{break} from loop \label{step:break}}
\ENDIF
\ENDWHILE
\IF{$\cT_0$ was not set yet}
\STATE If the algorithm ever went to the right, let $t_0$ be the last value for which this happened, and let $\cT_0 := \{t_0\}$. Otherwise, $\cT_0 := \emptyset$.\label{step:tzero}
\ENDIF
\STATE{Let $\cT_L$ be the set of all $t$ that were tested and made the search go left}
\STATE{Output $\hat{t} := \argmin_{t \in \cT_L \cup \cT_0 }  G(\hat{\epsilon}(t))$}

\end{algorithmic}	
\end{algorithm}
}

\newcommand{\estber}{\ensuremath{\algfont{EstBer}}}

To estimate the error, we sample random labeled examples from $\sinput$, and check the prediction error of $h^\nn_{\saughat\scale{t}}$ on these examples.
The prediction error of
any fixed hypothesis $h$
on a random labeled example from $\sinput$
is an independent Bernoulli variable with expectation
$\err(h,\sinput)$.
\esterr~is implemented using the following procedure, \estber, which
adaptively
estimates the expectation of a Bernoulli random
variable
to an accuracy specified by the
parameter $\theta$, using a small number of random independent Bernoulli experiments.
Let $B_1,B_2,\ldots \in \{0,1\}$ be i.i.d.~Bernoulli random variables. For an
integer $n$, denote $\hat{p}_n = \frac{1}{n}\sum_{i=1}^nB_i$.
The estimation procedure \estber\ is given in \algref{bernoulli}.
We prove a guarantee for this procedure in \thmref{bernstein}.
Note that we assume that the threshold parameter is in $(0,1]$, since for $\theta \geq 1$ one can simply output $1$ using zero random draws to satisfy \thmref{bernstein}.
\begin{algorithm}
\caption{\estber$(\theta,\beta,\delta)$}
\label{alg:bernoulli}
\begin{algorithmic}
\REQUIRE A threshold parameter $\theta \in (0,1]$, a budget parameter $\beta \geq 7$, confidence $\delta \in (0,1)$
\STATE $S \leftarrow \{B_1,\ldots,B_4\}$
\STATE $K \leftarrow \frac{4 \beta}{\theta}\log(\frac{8\beta}{\delta \theta})$
\FOR{$i = 3:\ceil{\log_2(\beta\log(2K/\delta)/\theta)}$} 
\STATE $n \leftarrow 2^i$
\STATE $S \leftarrow S \cup \{ B_{n/2+1},\ldots,B_{n}\}$.
\IF{$\hat{p}_n > \beta \log(2n/\delta)/n$}
\STATE{\textbf{break}}
\ENDIF
\ENDFOR
\STATE Output $\hat{p}_n$.
\end{algorithmic}
\end{algorithm}

The following theorem states that \algref{bernoulli} essentially estimates $p$, the expectation of the i.i.d.\ Bernoulli variables $B_1,B_2,\ldots$, up to a multiplicative constant, except if $p$ is smaller than a value proportional to the threshold $\theta$, in which case the algorithm simply returns a value at most $\theta$. Moreover, the theorem shows that the number of random draws required by the algorithm is inversely proportional to the maximum of the threshold $\theta$ and the expectation $p$. Thus, if $p$ is very small, the number of random draws does not increase without bound. The parameter $\beta$ controls the trade-off between the accuracy of estimation and the number of random draws. 
\begin{theorem}\label{thm:bernstein}
  Let $\delta \in (0,1)$, $\theta \in (0,1]$, $\beta \geq 7$. Let $B_1,B_2,\ldots \in \{0,1\}$ be i.i.d~Bernoulli random variables with expectation $p$. Let $p_o$ be the output of $\estber(\theta,\beta,\delta)$. The following holds with a probability of $1-\delta$, where  $f(\beta) := 1+ \frac{8}{3\beta} + \sqrt{\frac{2}{\beta}}$. 
\begin{enumerate}
\item If $p_o \leq \theta$, then $p \leq f(\beta)\theta$. Otherwise, $\frac{p}{f(\beta)} \leq p_o  \leq \frac{p}{2-f(\beta)}$.
\item Let $\psi := \max(\theta, p/f(\beta))$. The number of random draws in \estber\ is at most  $\frac{4\beta \log(\frac{8\beta }{\delta \psi})}{\psi}.$
\end{enumerate}
\end{theorem}
\begin{proof}
First, consider any single round $i$ with $n = 2^i$. By the empirical Bernstein bound \cite[Theorem 4]{DBLP:conf/colt/MaurerP09},
with a probability of $1-\delta/n$, for $n \geq 8$, we have\footnote{This follows from Theorem 4 of \cite{DBLP:conf/colt/MaurerP09} since $\frac{7}{3(n-1)} \leq \frac{8}{3n}$ for $n\geq 8$.}

\begin{equation}\label{eq:concent}
|\hat{p}_n - p| \leq \frac{8\log(2n/\delta)}{3n} + \sqrt{\frac{2\hat{p}_n \log(2n/\delta)}{n}}.
\end{equation}
Define $g := (\beta +  8/3 + \sqrt{2\beta})$, so that $f(\beta) = g/\beta$.
Conditioned on \eqref{concent}, there are two cases: 
\begin{enumerate}[label=(\alph*)]
\item $\hat{p}_n \leq \beta \log(2n/\delta)/n$. In this case, 
$p \leq g\log(2n/\delta)/n.$
\item $\hat{p}_n > \beta \log(2n/\delta)/n$. In this case, $n \geq \beta \log(2n/\delta)/\hat{p}_n$. Thus, by \eqref{concent}, 
\[
|\hat{p}_n - p| \leq \hat{p}_n (\frac{8}{3\beta} + \sqrt{2/\beta}) = \hat{p}_n(g/\beta -1).
\]
Therefore 
\[
\frac{\beta p}{g} \leq \hat{p}_n  \leq \frac{p}{2-g/\beta}.
\]
\end{enumerate}
Taking a union bound on all the rounds, we have that the guarantee holds for all rounds with a probability of at least $1-\delta$.

Condition now on the event that these guarantees all hold.
First, we prove the label complexity bound. Note that since $\beta \geq 7$, $K \geq 28$,
thus we have
$2\log(2K) > 8$, therefore there is always at least one round.
Let $n_o$ be the value of $n$ in the last round that the algorithm runs, and let $p_o = \hat{p}_{n_o}$.
Let $i$ such that $n_o = 2^{i+1}$, thus the algorithm stops during round $i+1$. 
This implies $\hat{p}_n \leq \beta \log(2n/\delta)/n$ for $n = 2^i$, therefore 
case (a) holds for $n$, which means $n \leq g \log(2n/\delta)/p$.
 It follows that $n \leq 2g \log(4g/\delta)/p$, therefore
$n_o \leq 4g\log(4g/(\delta p))/p.$ 
In addition, the number of random draws in the algorithm is $n_0$, which is bounded by 
\[
n_0 \leq 2^{\ceil{\log_2(\beta\log(2K/\delta)/\theta)}} \leq 2\cdot 2^{\log_2(\beta\log(2K/\delta)/\theta)} \leq 2\beta\log(2K/\delta)/\theta.
\]
Therefore we have the following bound on the number of random draws:
\[
n_o \leq \min\paren{\frac{2\beta\log(2K/\delta)}{\theta},\frac{4g\log(4g/(\delta p))}{p}}.
\]
Plugging in the definition of $K$ and substituting $\beta \cdot f(\beta)$ for $g$ yields
\begin{align*}
  n_0 \leq &\beta
  \min\paren{\frac{2\log(\frac{8\beta}{\delta\theta}\log(\frac{8\beta}{\delta \theta}))}{\theta},\frac{4 f(\beta)\log(\frac{4\beta f(\beta)}{\delta p})}{p}}
  \leq \\
  &\beta
  \min\paren{
    \frac{4\log(\frac{8\beta}{\delta\theta})}{\theta},
    \frac{4f(\beta)\log(\frac{4\beta f(\beta)}{\delta p})}{p}
    }
  \leq \\
  &4\beta \min
  \paren{\frac{1}{\theta} (\log(\frac{8\beta}{\delta})+\log(\frac{1}{\theta})),\frac{f(\beta)}{p}(\log(\frac{4\beta }{\delta}) + \log(\frac{f(\beta)}{p}))}
  .
\end{align*}
Using the definition of $\psi$, we get that the number of draws is at most
$\frac{4\beta \log(\frac{8\beta }{\delta \psi})}{\psi}.$

Next, we prove the accuracy of $p_o$ (item 1 in the theorem statement) by considering two cases. 
\begin{enumerate}[label=(\Roman*)]
\item If $p_o > \beta \log(2n_o/\delta)/n_o$,  then
case (b) above holds for $n_o$, 
thus
\[
\frac{\beta p}{g} \leq p_o  \leq \frac{p}{2-g/\beta}.
\]
In addition, if $p_o \leq \theta$, the LHS implies $p \leq f(\beta)\theta$. Thus item 1 in the theorem statement holds in this case.
\item If $p_o \leq \beta \log(2n_o/\delta)/n_o$, 
then \estber\ could not have ended by breaking out of the loop, thus it ran until the last round. Therefore $n_o \geq \beta\log(2K/\delta)/\theta$. In addition, case (a) holds for $n_0$, therefore
\begin{equation}\label{eq:tempbound}
p \leq \frac{g \log(2n_o/\delta)}{n_o} \leq \frac{g \theta \log(2n_0/\delta)}{\beta\log(2K/\delta)}.
\end{equation}
Now, for any possible value of $n_o$, 
\[
n_o\leq 2\beta\log(2K/\delta)/\theta \leq K.
\]
The first inequality follows from the bound on $i$ in \estber, and the second inequality holds since, as defined in \estber, $K \geq \frac{4\beta}{\theta}\log(\frac{8\beta}{\theta\delta})$. Since $n_0 \leq K$, \eqref{tempbound} implies that
\[
p \leq \frac{g\theta}{\beta} = f(\beta) \theta.
\]
In addition, we have
\[
p_o \leq \beta \log(2n_o/\delta)/n_o \leq \frac{ \theta \log(2n_0/\delta)}{\log(2K/\delta)} \leq \theta.
\]
Therefore in this case, necessarily $p_o \leq \theta$ and $p \leq f(\beta) \theta$, which satisfies item 1 in the theorem statement.
\end{enumerate}
In both cases item 1 holds, thus the theorem is proved.
\end{proof}

The procedure $\esterr(t,\theta,\delta)$ is then implemented as follows:
\begin{itemize}
\item Call $\estber(\theta,52,\delta/(2m^2))$, where the random variables $B_i$ are independent copies of the Bernoulli variable 
\[
B := \one[h^{\nn}_{\saughat\scale{t}}(X) \neq Y]
\] and $(X,Y) \sim \sinput$. 
\item To
  draw
  a single $B_i$, 
sample a random pair $(x',y')$ from $\sinput$, set 
\[
i := \kappa(x',\ConNet(\uinput,t/2)),
\]
 and get $S \leftarrow \GenNNSet(t,\{i\},\delta)$. This returns $S = ((x_i,\hat{y}_i))$ where $\hat{y}_i$ is the label of $x_i$ in $\saughat\scale{t}$. 
Then $B_i := \one[\hat{y}_i \neq y']$. Note that $B_i$ is indeed distributed like $B$, and $\E[B] = \epsilon(t)$. Note further that this call to $\GenNNSet(t,\{i\},\delta)$ uses $Q(m)$ label queries. Therefore the overall label complexity of a single draw of a $B_i$ is $Q(m)+1$. 
\end{itemize}

\corref{estimateerr} gives a guarantee for the accuracy and label complexity of \esterr. The proof is immediate from \thmref{bernstein}, by setting $\beta = 52$, which implies $f(\beta) \leq 5/4$.
\begin{corollary}\label{cor:estimateerr}
Let $t, \theta > 0$ and $\delta \in (0,1)$, and let  
$\hat{\epsilon}\scale{t} \leftarrow \esterr(t,\theta, \delta). $
Let $\qb$ as defined in  \eqref{qb}
The following properties hold with a probability of $1-\frac{\delta}{2m^2}$ over the randomness of $\esterr$ (and conditioned on $\saughat\scale{t}$). 
\begin{enumerate}
\item If $\hat{\epsilon}(t) \leq \theta$, then $\epsilon(t) \leq 5\theta/4$. Otherwise, 
\[
\frac{4 \epsilon(t)}{5} \leq \hat{\epsilon}(t)  \leq \frac{4\epsilon(t)}{3}.
\]
\item Let $\psi' := \max(\theta, \epsilon(t))$. The number of labels that \esterr\ requests is at most  
\[
 \frac{260 (\qb + 1)\log(\frac{1040 m^2 }{\delta \psi'})}{\psi'}.
\]
\end{enumerate}
\end{corollary}
To derive item $2.$ above from \thmref{bernstein}, note that for $\beta = 52$,
\[
\psi' = \max(\theta, \epsilon(t)) \leq f(\beta) \max(\theta, \epsilon(t)/f(\beta)) = f(\beta) \psi \leq \frac{5}{4}\psi,
\]
where $\psi$ is as defined in \thmref{bernstein}.
Below we denote the event that the two properties in \corref{estimateerr} hold for $t$ by $V(t)$.

\subsection{Selecting a Scale} \label{sec:binary}

The model selection procedure \selscale, given in \algref{tSearch}, implements its search based on the guarantees in \corref{estimateerr}.
First, we introduce some notation. We would like \algname\ to obtain a generalization guarantee that is competitive with $\gmin(m,\delta)$. 
Denote 
\begin{equation}\label{eq:phidef}
\phi\scale{t} := \frac{(\netsize\scale{t}+1)\log(m) + \log(\oo \delta)}{m},
\end{equation}
and let
\[
G(\epsilon,t) := \epsilon + \frac{2}{3} \phi\scale{t} + \frac{3}{\sqrt{2}}\sqrt{\epsilon \phi\scale{t}}.
\]
Note that for all $\epsilon,t$,
\[
\genbound(\epsilon,\netsize\scale{t},\delta,m,1) = \frac{m}{m-\netsize\scale{t}} G(\epsilon,t).
\]
When referring to $G(\nu\scale{t},t),G(\epsilon\scale{t},t)$, or $G(\hat{\epsilon}\scale{t},t)$ we omit the second $t$ for brevity.

Instead of directly optimizing $G(\nu(t))$, we will select a scale based on our estimate $G(\hat{\epsilon}\scale{t})$ of $G(\epsilon\scale{t})$.
Let $\dist$ denote the set of pairwise distances in the unlabeled dataset $\uinput$
(note that
$|\dist| < \binom{m}{2}$).
We remove from $\dist$ some distances, so that the remaining distances have a net size $\netsize\scale{t}$ that is monotone non-increasing in $t$. We also remove values with a very large net size. Concretely, define
\[
 \mdist := \dist \setminus \{ t \mid \netsize\scale{t}+1 > m/2\} \setminus \{t \mid \exists t' \in \dist, t' < t \text{ and } \netsize(t') < \netsize\scale{t}\} .
\]
Then for all $t,t' \in \mdist$ such that $t' < t$, we have $\netsize(t') \geq \netsize\scale{t}$.  The output of \selscale\ is always a value in $\mdist$. The following lemma shows that it suffices to consider these scales.

\begin{lemma}\label{lem:tmin}
Assume $m \geq 6$ and let $\tmon \in \argmin_{t \in \dist} G(\nu\scale{t})$. If $\gmin(m,\delta) \leq 1/3$ then 
$\tmon \in \mdist$.
\end{lemma}
\begin{proof}
Assume by way of contradiction that $\tmon \in \dist \setminus \mdist$.
First, since $G(\nu(\tmon)) \leq \gmin(m,\delta) \leq 1/3$ we have 
\[
\frac{\netsize(\tmon)+1}{m-\netsize(\tmon)}\log(m) \leq \half.
\] Therefore, since $m \geq 6$, it is easy to verify $\netsize(\tmon)+1 \leq m/2$. Therefore, by definition of $\mdist$ there exists a $t\leq \tmon$ with $\phi\scale{t} < \phi (\tmon)$.
Since $\nu\scale{t}$ is monotone over all of $t \in \dist$, we also have $\nu\scale{t} \leq \nu(\tmon)$.
Now, $\phi\scale{t} < \phi (\tmon)$ and $\nu\scale{t} \leq \nu(\tmon)$ together
imply that $G(\nu\scale{t})< G(\nu(\tmon))$, a contradiction.
Hence, $\tmon\in \mdist$.
\end{proof}

\selalg

\selscale\ follows a search procedure similar to binary search, however the conditions for going right and for going left are not exhaustive, thus it is possible that neither condition holds.
The search ends either when neither conditions hold, or when no additional scale should be tested. The final output of the algorithm is based on minimizing $G(\hateps\scale{t})$ over some of the values tested during search.

For $c > 0$, define 
\[
\gamma(c):= 1+\frac{2}{3c}+\frac{3}{\sqrt{2c}} \text{ and } \tilde\gamma(c):=\frac{1}{c}+\frac{2}{3}+\frac{3}{\sqrt{2c}}.
\] 
For all $t,\epsilon >0$ we have the implications
\begin{equation}\label{eq:impl}
 \epsilon \geq c \phi\scale{t} ~\Rightarrow~  \gamma(c)\epsilon \geq G(\epsilon,t)
\quad\text{ and }\quad 
 \phi\scale{t} \geq c \epsilon ~\Rightarrow~  \tilde\gamma(c)\phi\scale{t} \geq G(\epsilon,t).
\end{equation}
The following lemma uses \eqref{impl} to show that the estimate $G(\hateps\scale{t})$ is close to the true $G(\epsilon\scale{t})$. 

\begin{lemma}\label{lem:approxg}
Let $t > 0$, $\delta \in (0,1)$, and suppose that \selscale\ calls $\hateps(t) \leftarrow \esterr(t,\phi(t),\delta)$. Suppose that $V(t)$ as defined in \corref{estimateerr} holds. Then 
\[
\frac{1}{6} G(\hat{\epsilon}\scale{t}) \leq G(\epsilon\scale{t}) \leq 6.5 G(\hat{\epsilon}\scale{t}).
\]
\end{lemma}

\begin{proof}
Under $V(t)$, we have that if $\hat{\epsilon}\scale{t} < \phi\scale{t}$ then $\epsilon\scale{t} \leq \frac{5}{4} \phi\scale{t}$. In this case, 
\[
G(\epsilon\scale{t}) \leq \tilde\gamma(4/5) \phi\scale{t} \leq 4.3 \phi\scale{t},\] by \eqref{impl}. 
Therefore
\[
G(\epsilon\scale{t}) \leq \frac{3\cdot 4.3}{2} G(\hat{\epsilon}\scale{t}).
\] 
In addition, $G(\epsilon\scale{t}) \geq \frac{2}{3}\phi\scale{t} $ (from the definition of $G$), and by \eqref{impl} and $\tilde\gamma(1) \leq 4$, 
\[
\phi\scale{t} \geq \frac{1}{4} G(\hat{\epsilon}\scale{t}).
\] Therefore $G(\epsilon\scale{t}) \geq \frac{1}{6}G(\hat{\epsilon}\scale{t})$. 
On the other hand, if $\hat{\epsilon}\scale{t} \geq \phi\scale{t}$, then by \corref{estimateerr} 
\[
\frac{4}{5}\epsilon\scale{t} \leq \hateps\scale{t} \leq \frac43 \epsilon\scale{t}.
\]
Therefore
$G(\hateps\scale{t}) \leq \frac43 G(\epsilon\scale{t})$
and $G(\epsilon\scale{t}) \leq \frac54 G(\hateps\scale{t})$. Taking the worst-case of both possibilities, we get the bounds in the lemma.
\end{proof}
The next theorem bounds the label complexity of \selscale. Let $\ttest \subseteq \mdist$ be the set of scales that are tested during \selscale\ (that is, their $\hateps\scale{t}$ was estimated). 

\begin{theorem}\label{thm:selscalelabels}
  Suppose that the event $V(t)$ defined in \corref{estimateerr} holds for
  all $t \in \ttest$ for the calls $\hateps\scale{t} \leftarrow \esterr(t,\phi(t),\delta)$.  If the output of \selscale\ is
  $\hat{t}$, then the number of labels requested by \selscale\ is at most
\[
9620|\ttest| (\qb + 1)\frac{1}{G(\epsilon(\hat{t}))}\log(\frac{38480 m^2 }{\delta G(\epsilon(\hat{t}))}).
\]
\end{theorem}
\begin{proof}
The only labels used by the procedure are those used by calls to \esterr.
Let $\psi_t := \max(\phi\scale{t},\epsilon\scale{t})$, and $\psi_{\min} := \min_{t \in \ttest} \psi_t$. Denote also $\hat{\psi}_t := \max(\phi\scale{t},\hat{\epsilon}\scale{t})$.
 From \corref{estimateerr} we have that the total number of labels in all the calls to \esterr\ in \selscale\ is at most 
\begin{equation}\label{eq:initbound}
\sum_{t \in \ttest}\frac{260 (\qb + 1)\log(\frac{1040 m^2 }{\delta \psi_t})}{\psi_t} \leq |\ttest| \frac{260 (\qb + 1)\log(\frac{1040 m^2 }{\delta \psi_{\min}})}{\psi_{\min}}.
\end{equation}
We now lower bound $\psi_{\min}$ using $G(\epsilon(\hat{t}))$. By \lemref{approxg} and
the choice of $\hat{t}$,
\[
G(\epsilon(\hat{t})) \leq 6.5 G(\hateps(\hat{t})) = 6.5 \min_{t \in \cT_L
  \cup T_0}G(\hateps\scale{t}).
\]
From the definition of $G$, for any $t > 0$,
\[
G(\hateps\scale{t}) \leq \gamma(1)\max(\phi\scale{t}, \hateps\scale{t}) \leq 25 \hat\psi_t.
\]
 Therefore
\begin{equation}\label{eq:gleqmin}
G(\epsilon(\hat{t})) \leq 25 \min_{t \in \cT_L  \cup \cT_0}\hat{\psi}_t.
\end{equation}
We will show a similar upper bound when minimizing over all of $\ttest$, not just over $\cT_L\cup \cT_0$. 
This is trivial if $\ttest = \cT_L\cup \cT_0$. Consider the case $\cT_L \cup \cT_0 \subsetneq \ttest$.
For any $t \in \ttest$, we have one of:
\begin{itemize}
\item The search went left on $t$ (step \ref{step:left}), hence $t \in \cT_L$.
\item The search went nowhere on $t$ and the loop broke (step \ref{step:break}), hence
  $t = t_0 \in \cT_0$.
\item The search went right on $t$ (step \ref{step:right}) and this was the last value for which this happened, hence $t = t_0 \in \cT_0$.
\item The search went right on $t$ (step \ref{step:right}) and this was \emph{not} the last value for which this happened. Hence $t \in \ttest \setminus (\cT_L \cup \cT_0)$.
\end{itemize}
Set some $t_1 \in \ttest \setminus (\cT_L \cup \cT_0)$. Since the search went right on $t_1$, then $t_0$ also exists, since the algorithm did go to the right for some $t$ (see step \ref{step:tzero}).
Since the binary search went right on $t_1$, we have $\hat{\epsilon}(t_1) \leq \phi(t_1)$. Since the binary search did \emph{not} go left on $t_0$ (it either broke from the loop or went right), $\hat{\epsilon}(t_0) \leq \frac{11}{10}\phi(t_0)$. 

In addition, $t_0 \geq t_1$ (since the search went right at $t_1$, and $t_0$ was tested later than $t_1$), thus $\phi(t_0) \leq \phi(t_1)$ (since $t_0,t_1 \in \mdist$). Therefore,
\[
\hat{\psi}_{t_0} = \max(\phi(t_0),\hat{\epsilon}(t_0)) \leq \frac{11}{10}\phi(t_0) \leq \frac{11}{10}\phi(t_1) = \frac{11}{10}\max(\phi(t_1),\hat{\epsilon}(t_1)) = \hat{\psi}_{t_1}.
\]
It follows that for any such $t_1$, 
\[
\min_{t \in \cT_L \cup \cT_0}\hat\psi_{t} \leq\frac{11}{10}\hat{\psi}_{t_1}.
\]
Therefore
\[
\min_{t \in \cT_L \cup \cT_0}\hat{\psi}_{t} \leq \frac{11}{10}\min_{t \in \ttest} \hat{\psi}_{t}.
\]
Therefore, by \eqref{gleqmin}
\[
G(\epsilon(\hat{t})) \leq 27.5 \min_{t \in \ttest} \hat{\psi}_t.
\]
By \corref{estimateerr}, $\hat{\epsilon}\scale{t} \leq \max(\phi\scale{t}, 4\epsilon\scale{t}/3)$, therefore $\hat{\psi}_t \leq \frac{4}{3}\psi_t$. Therefore $G(\epsilon(\hat{t})) \leq 37\psi_{\min}.$ 
Therefore, from \eqref{initbound}, the total number of labels is at most
\[
9620|\ttest| (\qb + 1)\frac{1}{G(\epsilon(\hat{t}))}\log(\frac{38480 m^2 }{\delta G(\epsilon(\hat{t}))}).
\]
\end{proof}

The following theorem provides a competitive error guarantee for the selected scale $\hat{t}$.

\begin{theorem}\label{thm:selscalegen}
Suppose that $V(t)$ and $E(t)$, defined in \corref{estimateerr} and \thmref{nnset}, hold for all values $t \in \ttest$, and that $\gmin(m,\delta) \leq 1/3$.
Then \selscale\ outputs $\hat{t} \in \mdist$ such that
\[
\genbound(\epsilon(\hat{t}),\netsize(\hat{t}),\delta,m,1) \leq O(\gmin(m,\delta)),
\] 
Where the $O(\cdot)$ notation hides only universal multiplicative constants.
\end{theorem}

The full proof of this theorem is given below. The idea of the proof is as follows: First, we show (using \lemref{approxg}) that it suffices to prove that $G(\nu(\tmon)) \geq O(G(\hat{\epsilon}(\hat{t})))$ to derive the bound in the theorem. Now, \selscale\ ends in one of two cases: either $\cT_0$ is set within the loop, or $\cT = \emptyset$ and $\cT_0$ is set outside the loop. In the first case, neither of the conditions for turning left and turning right holds for $t_0$, so we have $\hat{\epsilon}(t_0) = \Theta(\phi(t_0))$ (where $\Theta$ hides numerical constants). We show that in this case, whether $\tmon \geq t_0$ or $\tmon \leq t_0$, $G(\nu(\tmon)) \geq O(G(\hat{\epsilon}(t_0)))$. In the second case, there exist (except for edge cases, which are also handled) two values $t_0 \in \cT_0$ and $t_1 \in \cT_L$ such that $t_0$ caused the binary search to go right, and $t_1$ caused it to go left, and also $t_0 \leq t_1$, and $(t_0,t_1) \cap \mdist = \emptyset$. We use these facts to show 
that for $\tmon \geq t_1$, $G(\nu(\tmon)) \geq O(G(\hat{\epsilon}(t_1)))$, and for $\tmon \leq t_0$, $G(\nu(\tmon)) \geq O(G(\hat{\epsilon}(t_0)))$. Since $\hat{t}$ minimizes over a set that includes $t_0$ and $t_1$, this gives $G(\nu(\tmon)) \geq O(G(\hat{\epsilon}(\hat{t})))$ in all cases.

\vspace{1em}
\begin{proof}
First, note that it suffices to show that there is a constant $C$, such that for the output $\hat{t}$ of \selscale, we 
have $G(\epsilon(\hat{t})) \leq C G(\nu(\tmon))$. 
This is because of the following argument: From \lemref{tmin} we have that if $\gmin(m,\delta) \leq 1/3$, then $\tmon\in \mdist$. 
Now
\[
 \gmin(m,\delta) = \frac{m}{m- \netsize(\tmon)} G(\nu(\tmon))  \geq G(\nu(\tmon)).
\]
And, if we have the guarantee on $G(\epsilon(\hat{t}))$ and $\gmin(m,\delta) \leq 1/3$ we will have
\begin{equation}\label{eq:genboundc}
\genbound(\epsilon({\hat{t}}), \netsize(\hat{t}), \delta, m, 1) = \frac{m}{m - \netsize(\hat{t})} G(\epsilon(\hat{t})) \leq 2G(\epsilon(\hat{t})) \leq 2C G(\nu(\tmon)) \leq 2C \gmin(m,\delta).
\end{equation}

We now prove the existence of such a guarantee and set $C$.
Denote the two conditions checked in \selscale\  during the binary search by Condition 1: $\hateps\scale{t} < \phi\scale{t}$  and Condition 2: $\hateps\scale{t} > \frac{11}{10}\phi\scale{t}$.
The procedure ends in one of two ways: either $\cT_0$ is set within the loop (Case 1), or $\cT = \emptyset$ and $\cT_0$ is set outside the loop (Case 2). We analyze each case separately.

In Case 1, none of the conditions 1 and 2 hold for $t_0$. Therefore 
\[\phi(t_0) \leq \hateps(t_0) \leq \frac{11}{10} \phi(t_0).\] 
Therefore, by \eqref{impl},
\[\phi(t_0)  \geq  G(\hateps(t_0))/\tilde\gamma(\frac{10}{11}).\]
By \corref{estimateerr}, since $\hateps(t_0) > \phi(t_0)$, 
\[
\frac{3}{4} \phi(t_0) \leq \frac{3}{4} \hateps(t_0) \leq \epsilon(t_0) \leq \frac{5}{4} \hateps(t_0) \leq \frac{55}{40} \phi(t_0).
\]
Suppose $\tmon \geq t_0$, then 
\[
G(\nu(\tmon)) \geq \nu(\tmon) \geq \nu(t_0) \geq \frac{1}{4} \epsilon(t_0) \geq \frac{3}{16} \phi(t_0).
\]
here we used $\epsilon(t_0) \leq 4\nu(t_0)$ by \thmref{nnset}.
Therefore, from \eqref{impl} and \lemref{approxg},
\[
G(\nu(\tmon)) \geq \frac{3}{16}\phi(t_0) \geq \frac{3}{16\tilde\gamma\left(\frac{40}{55}\right)}G(\epsilon(t_0)) \geq \frac{\half}{16\tilde\gamma(\frac{40}{55})}G(\hateps(t_0)).
\]
Now, suppose $\tmon < t_0$, then
\[
G(\nu(\tmon)) \geq \frac{2}{3}\phi(\tmon) \geq \frac{2}{3}\phi(t_0) \geq \frac{2}{3\tilde\gamma(\frac{10}{11})}G(\hateps(t_0)).
\]
In this inequality we used the fact that $\tmon,t_0 \in \mdist$, hence $\phi(\tmon) \geq \phi(t_0)$. Combining the two possibilities for $\tmon$, we have in Case 1,
\[
G(\hateps(t_0)) \leq \max(32\tilde\gamma(\frac{40}{55}),\frac{3\tilde\gamma(\frac{10}{11})}{2}) G(\nu(\tmon)).
\]
Since $\hat{t}$ minimizes $G(\hateps\scale{t})$ on a set that includes $t_0$, we have, using \lemref{approxg} 
\[
G(\epsilon(\hat{t})) \leq 6.5 G(\hateps(\hat{t})) \leq 6.5 G(\hateps(t_0)).
\]
Therefore, in Case 1, 
\begin{equation}\label{eq:caseone}
G(\epsilon(\hat{t})) \leq 6.5\max(32\tilde\gamma(\frac{40}{55}),\frac{3\tilde\gamma(\frac{10}{11})}{2})G(\nu(\tmon)).
\end{equation}

In Case 2, the binary search halted without satisfying Condition 1 nor Condition 2 and with $\cT = \emptyset$. Let $t_0$ be as defined in this case in \selscale\ (if it exists), and let $t_1$ be the smallest value in $\cT_L$ (if it exists). At least one of these values must exist. If both values exist, we have $t_0 \leq t_1$ and $(t_0,t_1) \cap \mdist = \emptyset$. 

If $t_0$ exists, it is the last value for which the search went right. We thus have $\hateps(t_0) < \phi(t_0)$. If $\tmon \leq t_0$, from condition 1 on $t_0$ and \eqref{impl} with $\tilde\gamma(1) \leq 4$,
\[
G(\nu(\tmon)) \geq \frac{2}{3}\phi(\tmon) \geq \frac{2}{3}\phi(t_0) \geq  \frac{1}{6}G(\hateps(t_0)).
\]
Here we used the monotonicity of $\phi$ on $\tmon,t_0 \in \mdist$, and \eqref{impl} applied to condition 1 for $t_0$. 

If $t_1$ exists, the search went left on $t_1$, thus $\hateps(t_1) > \frac{11}{10}\phi(t_1)$. By \corref{estimateerr}, it follows that $\hateps(t_1) \leq \frac43\epsilon(t_1)$. Therefore,
if $\tmon \geq t_1$, 
\[
G(\nu(\tmon)) \geq \nu(\tmon) \geq \nu(t_1) \geq \frac{1}{4}\epsilon(t_1) \geq \frac{3}{16}\hateps(t_1) \geq \frac{3}{16\gamma(11/10)}G(\hateps(t_1)).
\]
Here we used $\epsilon(t_1) \leq 4\nu(t_1)$ by \thmref{nnset} and \eqref{impl}.
Combining the two cases for $\tmon$, we get that if $t_0$ exists and $\tmon \leq t_0$, or $t_1$ exists and $\tmon \geq t_1$, 
\[
G(\nu(\tmon)) \geq \min(\frac{1}{6},\frac{3}{16\gamma(11/10)})\min_{t\in T_E}G(\hateps\scale{t}).
\]
where we define $T_E = \{t \in \{t_0,t_1\} \mid t \text{ exists}\}$.
We now show that this covers all possible values for $\tmon$:
If both $t_0,t_1$ exist, then since $(t_0,t_1) \cap \mdist = \emptyset$, it is impossible to have $\tmon \in (t_0,t_1)$.
If only $t_0$ exists, then the search never went left, which means $t_0 = \max (\mdist)$, thus $\tmon \leq t_0$. If only $t_1$ exists, then the search never went right, which means $t_1 = \min (\mdist)$, thus $\tmon \geq t_1$.

Since $\hat{t}$ minimizes $G(\hateps\scale{t})$ on a set that has $T_E$ as a
subset, we have, using \lemref{approxg} $G(\epsilon(\hat{t})) \leq 6.5 G(\hateps(\hat{t})) \leq 6.5\min_{t\in T_E}G(\hateps\scale{t}).$
Therefore in Case 2, 
\begin{equation}\label{eq:casetwo}
G(\nu(\tmon)) \geq \frac{1}{6.5}\min\paren{\frac{1}{6},\frac{3}{16\gamma(11/10)}} G(\epsilon(\hat{t}).
\end{equation}

From \eqref{caseone} and \eqref{casetwo} we get that in both cases
\[
G(\nu(\tmon)) \geq \frac{1}{6.5}
\min
\paren{
  \frac{1}{6},\frac{3}{16\gamma(11/10)},\frac{2}{3\tilde\gamma(10/11)},\frac{1}{32\tilde\gamma(\frac{40}{55})}
}
G(\epsilon(\hat{t})) \geq G(\epsilon(\hat{t}))/865.
\]
Combining this with \eqref{genboundc} we get the statement of the theorem.
\end{proof}

\section{Bounding the Label Complexity of \algname} \label{sec:proofmain}

We are now almost ready to prove \thmref{main}. Our last missing piece is quantifying the effect of side information on the generalization of sample compression schemes in \secref{compress}. We then prove \thmref{main} in \secref{subproofmain}.

\subsection{Sample Compression with Side Information}\label{sec:compress}

It appears that compression-based generalization bounds were
independently discovered by
  \citet{warmuth86} and \citet{MR1383093}; some background is given in
  \citet{DBLP:journals/ml/FloydW95}. As noted in \secref{gennn}, our algorithm relies on a generalized sample compression scheme, which requires side information. This side information is used to represent the labels of the sample points in the compression set. 
A similar idea appears in \citet{DBLP:journals/ml/FloydW95} for hypotheses with short description length. Here we provide a generalization that is useful for the analysis of \algname.

  Let $\Sigma$ be a finite alphabet, and define a mapping
$\Rec_\seqN: (\cX \times \cY)^\seqN \times \Sigma^\seqN \rightarrow \cY^\cX$.\footnote{If $\cX$ is infinite, replace $\cY^\cX$ with the set of measurable functions from $\cX$ to $\cY$.}
This is a {\em reconstruction} function
mapping
a labeled
sequence of size
$\seqN$
with side information $T\in\Sigma^\seqN$ to a classifier. 
For $I\subseteq[|S|]$, denote by $S[I]$ the subsequence of $S$
indexed by $I$.
For a labeled sample $S$, define the set of possible hypotheses reconstructed from a compression of $S$ of size $N$ with side information in $\Sigma$: 
$
\cH_\seqN(S) := \set{   h:\cX \rightarrow \cY \mid  h =
  \Rec_\seqN(S[I],T), I \in [m]^\seqN, T \in \Sigma^\seqN}
$. 
The following result closely follows the sample compression arguments in
\citet[Theorem 2]{DBLP:journals/ml/GraepelHS05},
and 
\citet[Theorem 6]{gkn-jmlr17+aistats}, but incorporates side information.
\begin{theorem}
  \label{thm:compression}
  Let $m$ be an integer and $\delta \in (0,1)$.
  Let $S \sim \cD^m$. With probability at least $1-\delta$,
  if there exist $\seqN < m$ and $h \in \cH_\seqN(S)$
  with $\epsilon := \err(h,S) \leq \half$, then
  $\err(h,\cD) \leq \genbound(\epsilon,\seqN,\delta,m, |\Sigma|).$
\end{theorem}
\begin{proof}
  We recall a result of
  \citet[Lemma 1]{DBLP:conf/icml/DasguptaH08}:
if $\hat p\sim\bin(n,p)/n$ and $\delta>0$,
then
the following
holds with probability at least $1-\delta$:
\begin{equation}
\label{eq:emp-bern}
p \le \hat p + 
\frac{2}{3n}\log\frac{1}\delta
+
\sqrt{ \frac{9\hat p(1-\hat p)}{2n}\log\frac{1}\delta}
.
\end{equation}
Now fix $N < m$, and suppose that
$h\in\cH_\netsize(S)$ has $\hat{\epsilon} \leq \half$. 
Let $I \in [m]^N, T \in \Sigma^N$ such that $h = \Rec_\seqN(S[I],T)$. 
We have $\err(h,S[[m] \setminus I]) \leq \frac{\hat{\epsilon} m}{m - N} = \theta \hat{\epsilon}$. Substituting into (\ref{eq:emp-bern}) $p:=\err(h,\cD)$, $n := m-N$ and $\hat{p}: = \err(h,S[[m] \setminus I]) \leq \theta \hat{\epsilon}$,
yields that for a fixed $S[I]$ and a random $S[ [m] \setminus I] \sim \cD^{m-N}$,  with probability at least $1-\delta$,
\begin{equation}
\label{eq:err-serr}
\err(h,\cD)
\le 
\theta \hat{\epsilon} + 
\frac{2}{3(m-N)}\log\frac{1}\delta
+
\sqrt{ \frac{9\theta \hat{\epsilon}}{2(m-N)}\log\frac{1}\delta}.
\end{equation}
To make (\ref{eq:err-serr}) hold simultaneously for all
$(I,T)\in[m]^N\times\Sigma^N$,
divide $\delta$ by $(m|\Sigma|)^N$.
To make the claim hold for all $N\in[m]$, stratify (as in \citet[Lemma 1]{DBLP:journals/ml/GraepelHS05}) over the (fewer than)
$m$ possible choices of $N$, which amounts to dividing
$\delta$
by an additional factor of $m$. 
\end{proof}

For \algname, we use the following
sample compression scheme
with $\Sigma = \cY$.
Given a subsequence $S' := S[I] := (x'_1,\ldots,x'_\seqN) $
and
$T = (t_1,\ldots,t_\seqN) \in \cY^\seqN$, the reconstruction function
$\Rec_\seqN(S[I],T)$ generates the nearest-neighbor rule
induced by
the labeled sample
$\psi(S',T) := ((x'_i,t_i))_{i \in [\seqN]}$.
Formally, $\Rec_\seqN(S',T)= h_{\psi(S',T)}^\nn$. Note the slight abuse of notation:
  formally, the $y_i$ in $\saug\scale{t}$ should be encoded as side information $T$,
  but for clarity, we have opted to ``relabel'' the examples $\{x_1,\ldots,x_\seqN\}$
  as dictated by the majority in each region.
The following corollary is immediate from \thmref{compression} and the construction above.
\begin{theorem}\label{thm:activecompression}
  Let $m \geq |\cY|$ be an integer, $\delta \in (0,\oo4)$. Let $\sinput \sim \cD^m$.
  With probability at least $1-\delta$,
  if there exist $\seqN < m$ and
  $S \subseteq (\cX \times \cY)^\seqN$
  such that 
  $\uu(S) \subseteq \uinput$ and
  $\epsilon := \err(h^\nn_{S},\sinput) \leq \half$, then 
$\err(h^\nn_{S},\cD) \leq
\genbound(\epsilon, \seqN, \delta, m, |\cY|) \leq
2\genbound(\epsilon,\seqN,2\delta,m,1).$
\end{theorem}
If the compression set includes only the original labels,
the compression analysis of \citet{gkn-jmlr17+aistats}
gives the bound $\genbound(\epsilon, \seqN, \delta, m, 1)$.
Thus the effect of allowing the labels to change is only logarithmic in $|\cY|$, and does not appreciably degrade the prediction error. 

\subsection{Proof of \thmref{main}}\label{sec:subproofmain}

The proof of the main theorem, \thmref{main}, which gives the guarantee for \algname, is almost immediate from \thmref{activecompression}, \thmref{nnset}, \thmref{selscalegen} and \thmref{selscalelabels}. 

\begin{proof}[of \thmref{main}]
  We have $|\mdist| \leq \binom{m}{2}$. By a union bound, the events $E(t)$ and $V(t)$ of \thmref{nnset} and \corref{estimateerr} hold for all $t \in \ttest \subseteq \mdist$ with a probability of at least $1-\delta/2$. Under these events, we have by \thmref{selscalegen} that if $\gmin(m,\delta) \leq 1/3$,
\[
\genbound(\epsilon(\hat{t}),\netsize(\hat{t}),\delta,m,1) \leq O\left(\min_{t} \genbound(\nu\scale{t},\netsize\scale{t},\delta,m,1)\right).
\]
By \thmref{activecompression}, with a probability at least $1-\delta/2$, if
$\epsilon(\hat{t}) \leq \half$ then

\[
\err(\hat{h},\cD) \leq
2\genbound(\epsilon(\hat{t}),\netsize(\hat{t}),\delta,m,1).
\]

The statement of the theorem follows. Note that the statement trivially holds
for $\gmin(m,\delta) \geq 1/3$ and for $\epsilon(\hat{t}) \geq \half$, thus these
conditions can be removed.  To bound the label complexity, note that the total
number of labels used by \algname\ is at most the number of labels used by
\selscale\ plus the number of labels used by $\GenNNSet$ when the final
compression set is generated. 

By \thmref{selscalelabels}, since $\qb = O(\log(m/\delta))$, the number of labels used by
\selscale\ is at most
\[
O\left(|\ttest|\frac{\log^2(m/\delta)}{G(\epsilon(\hat{t}))}\log\left(\frac{1}{G(\epsilon(\hat{t})}\right)\right).
\]
In addition,
\[
G(\epsilon(\hat{t})) \geq
\genbound(\epsilon(\hat{t}),\netsize(\hat{t}),\delta,m,1) = \hat{G}.
\]
The number of tested scales in \selscale\ is bounded by
\[
|\ttest| \leq \floor{\log_2(|\mdist|)+1} \leq 2\log_2(m) 
\]
Therefore the number of labels used by \selscale\ is 
\[
O\left(\frac{\log^3(m/\delta)}{\hat{G}}\log\left(\frac{1}{\hat{G}}\right)\right).
\]
The number of labels used by $\GenNNSet$ is at most $\qb \netsize(\hat{t})$, where $\qb \leq O(\log(m/\delta)$, 
and from the definition of $\hat{G}$, $\netsize(\hat{t}) \leq O(m \hat{G}/\log(m))$.  Summing up the number of labels used by \selscale\ and the number used by $\GenNNSet$, this gives the bound in the statement of the theorem.
\end{proof}

\section{Passive Learning Lower Bounds}\label{sec:passivelower}
\newcommand{\Rad}[1]{\operatorname{R}_{#1}}
\newcommand{\kl}{\operatorname{KL}}
\newcommand{\LB}{\operatorname{{LB}}}
\newcommand{\Bin}{\operatorname{{Bin}}}
\newcommand{\smv}{^{\textrm{{\tiny \textup{SMV}}}}}
\newcommand{\opt}{\operatorname{bayes}}
\newcommand{\hatopt}{\operatorname{ba\check y es}}
\newcommand{\polylog}{\operatorname{polylog}}
\newcommand{\gn}{\, | \,}
\newcommand{\unif}{\mathrm{Unif}}
\newcommand{\pred}[1]{\one[ #1 ]}
\newcommand{\nrm}[1]{\left\Vert #1 \right\Vert}
\newcommand{\Ds}{\cD_{\sigma,b,p}}
\newcommand{\hl}{\hat h_\ell}
\newcommand{\noise}{{\eta}}
\renewcommand{\E}{\mathop{\mathbb{E}}}
\renewcommand{\P}{\mathop{\mathbb{P}}}
\newcommand{\setzo}{\set{0,1}}
\newcommand{\sD}{\mathscr{D}}

\thmref{passivelower} lower bounds the performance of a passive learner who
observes a limited number $\ell$ of random labels from $\sinput$. The number $\ell$
is chosen so that it is of the same order as the number of labels \algname\
observes for the case analyzed in \secref{main}.
We deduce Theorem~\ref{thm:passivelower} from a more general result pertaining to the sample complexity of passive learning. The general result is given as \thmref{passivelower-gen} in \secref{genpassivelower}. The proof of \thmref{passivelower} is provided in \secref{passivelowerproof}.

We note that while the lower bounds below assume that the passive learner observes only the random labeled sample of size $\ell$, in fact their proofs hold also if the algorithm has access to the full unlabeled sample of size $m$ of which $S_\ell$ is sampled. This is because the lower bound is based on requiring the learner to distinguish between distributions that all have the same marginal. Under this scenario, access to unlabeled examples does not provide any additional information to the learner. 

\subsection{A General Lower Bound}\label{sec:genpassivelower}
In this section we show a general sample complexity lower bound for passive learning, which may be of independent interest. We are aware of two existing lower bounds for agnostic PAC with bounded
Bayes error:
  \citet[Theorem 14.5]{MR1383093}
  and
  \citet[Theorem 8.8]{audibert2009}.
  Both place restrictions on the relationship between the sample size, VC-dimension, and
  Bayes error
  level,
  which render them inapplicable
  as stated to some parameter regimes, including the one needed for proving \thmref{passivelower}.
  
Let $\cH$ be a hypothesis class
with VC-dimension $d$ and suppose that $\cL$
is a passive learner\footnote{
  We allow $\cL$ access to an independent internal source of randomness.
}
mapping labeled samples
$S_\ell=(X_i,Y_i)_{i\in[\ell]}$ to hypotheses
$\hl\in\cH$.
For any distribution
$\cD$ over $\cX\times\setpm$,
define
the {\em excess risk} of $\hl$ by
\beq
\Delta(\hl,\cD)
:=
\err(\hl,\cD)-
\inf_{h\in\cH}\err(h,\cD)
.
\eeq
Let $\sD(\noise)$ be the collection of all
{\em $\noise$-bounded agnostic error}
distributions $\cD$
over $\cX\times\setpm$
that satisfy
$
\inf_{h\in\cH}\err(h,\cD)
\le\noise$.
We say that $Z\in\setpm$ has Rademacher distribution with parameter
$b\in[-1,1]$,
denoted $Z\sim\Rad{b}$,
if
\beq
\P[Z=1]=1-\P[Z=-1] = \oo2+\frac b2.
\eeq
All distributions
on $\setpm$ are of this form.
For $k\in\nats$ and $b\in[0,1]$, define the function
\beq
\opt(k,b)
= \hf\paren{1-\hf\nrm{\Rad{b}^k-\Rad{-b}^k}_1}
,
\eeq
where $\Rad{\pm b}^k$ is the corresponding product distribution on $\setpm^k$
and $\hf\nrm{\cdot}_1$ is the total variation norm.
This expression previously appeared in 
\citet[Equation (25)]{DBLP:journals/jmlr/BerendK15}
in the context of information-theoretic lower bounds;
the current terminology
was motivated in 
\citet{DBLP:journals/corr/KontorovichP16},
where various precise estimates on $\opt(\cdot)$
were provided.
In particular, the function $\hatopt(\kappa,b)$ was defined as follows:
for each fixed $b\in[0,1]$,
$\hatopt(\cdot,b)$
is the largest convex minorant
on $[0,\infty)$ of the function $\opt(\cdot,b)$ on $\{0,1,\dots\}$.
  It was shown in
  \citet[Proposition 2.8]{DBLP:journals/corr/KontorovichP16}
  that $\hatopt(\cdot,b)$ is the linear interpolation of $\opt(\cdot,b)$ at the points $0,1,3,5,\dots$.

\begin{theorem}
  \label{thm:passivelower-gen}  
Let
  $0<\noise<\hf$,
$\ell\ge1$, and $\cH$
be a 
hypothesis class
with VC-dimension $d>1$.
Then,
for all $0<b,p<1$
satisfying
\beqn
\label{eq:bpnu}
p\paren{\oo2-\frac b2}
\le \noise,
\eeqn
we have
\beqn
\label{eq:main-lb-E}
\inf_{\hl} \sup_{\cD\in\sD(\noise)}
\E_{\cD^\ell}\sqprn{
  \Delta(\hl,\cD)
}
&\ge&
bp
\hatopt(\ell p/(d-1),b).
\eeqn
Furthermore,
for $0\le u<1$,
\beqn
\label{eq:main-lb-P}
\inf_{\hl} \sup_{\cD\in\sD(\noise)} \P\sqprn{\Delta(\hl,\Ds)>bpu}
&>&
\hatopt(\ell p/(d-1),b)
-u.
\eeqn
\end{theorem}  
\begin{proof}
This proof uses ideas from \citet[Theorem 14.5]{MR1383093}, \citet[Theorem 5.2]{MR1741038} and
\citet[Theorem 2.2]{DBLP:journals/corr/KontorovichP16}.

We will construct adversarial distributions
supported on a shattered subset of size $d$,
and hence there is no loss of generality in
taking $\cX=[d]$ and $\cH=\setpm^\cX$.
A random distribution
$\Ds$
over $\cX\times\setpm$,
parametrized by
a random
$\sigma\sim\unif(\setpm^{d-1})$
and scalars $b,p\in(0,1)$ to be specified later,
is defined
as follows.
The point $x=d\in\cX$ gets a marginal weight of $1-p$
where $p$ is a parameter to be set;
the remaining $d-1$ points each get a marginal weight of $p/(d-1)$:
\beqn
\label{eq:D(X)}
\P_{X\sim\Ds}[X=d]=1-p,\qquad
\P_{X\sim\Ds}[X<d]=\frac{p}{d-1}.
\eeqn
The distribution of $Y$ conditional on $X$ is
given by
$
\P
_{(X,Y)\sim\Ds}
[Y=1\gn X=d]=1
$
and
\beqn
\label{eq:D(Y|X)}
\P
_{(X,Y)\sim\Ds}
[Y=\pm1 \gn X=j<d]
=
\oo2\pm\frac{b\sigma_j}2
.
\eeqn
Suppose that $(X_i,Y_i)_{i\in[\ell]}$ is a sample drawn from
$\Ds^\ell$.
The assumption that $\Ds\in\sD(\noise)$
implies that
$b$ and $p$
must satisfy the constraint (\ref{eq:bpnu}).

A standard argument
(e.g., \citet{MR1741038} p. 63 display (5.5))
shows that,
for any hypothesis $\hl$,
\beqn
\Delta(\hl,\Ds)
&=&
\err(\hl,\Ds)-
\inf_{h\in\cH}\err(h,\Ds)
\nonumber\\
&=&
\P_{X\sim\Ds}[X=d,\hl(X)\neq1]
+
b
\P_{X\sim\Ds}[X<d,\hl(X)\neq \sigma(X)]
\nonumber\\
&\ge&
b\P_{X\sim\Ds}[X<d,\hl(X)\neq \sigma(X)]
\nonumber\\
&=&
\label{eq:Delta-lb}
bp\P_{X\sim\Ds}[\hl(X)\neq \sigma(X)|X<d]
.
\eeqn

It follows from
\citet[Theorems 2.1, 2.5]{DBLP:journals/corr/KontorovichP16}
that
\beqn
\label{eq:KP}
\E_{\sigma
}
\P_{X\sim\Ds}[\hl(X)\neq \sigma(X)|X<d]
&\ge&
\E_{N\sim\Bin(\ell,p/(d-1))} [\opt(N,b)]
\\\nonumber
&\ge&
\E_{N\sim\Bin(\ell,p/(d-1))} [\hatopt(N,b)]
\\\nonumber
&\ge&
\hatopt(\E[N],b) = \hatopt(\ell p/(d-1),b),
\eeqn
where the second inequality holds because
$\hatopt$ is, by definition, a convex minorant of $\opt$,
and the third follows from Jensen's inequality.
Combined with
(\ref{eq:Delta-lb}),
this proves (\ref{eq:main-lb-E}).

To show (\ref{eq:main-lb-P}),
define the random variable
\beq
Z=Z(\sigma,\cL)=\P_{X\sim\Ds}[\hl(X)\neq \sigma(X)|X<d]
.
\eeq
Since $Z\in[0,1]$,
Markov's inequality implies
\beq
\P[Z>u]
\ge\frac{\E[Z]-u}{1-u}
>\E[Z]-u,\qquad 0\le u<1.
\eeq
Now (\ref{eq:Delta-lb}) implies that
$\Delta(\hl,\Ds)
\ge bpZ$
and hence, for $0\le u<1$,
\beq
\inf_{\hl} \sup_{\cD\in\sD(\noise)} \P\sqprn{\Delta(\hl,\Ds)>bpu}
&=&
\inf_{\hl} \sup_{\cD\in\sD(\noise)}\P[Z>u] \\
&>&
\inf_{\hl} \sup_{\cD\in\sD(\noise)}\E[Z]-u \\
&\ge& \oo{bp}
\inf_{\hl} \sup_{\cD\in\sD(\noise)}\E[\Delta(\hl,\Ds)]-u
\\
&\ge&
\hatopt(\ell p/(d-1),b)
-u.
\eeq
\end{proof}

\subsection{Proof of Theorem~\ref{thm:passivelower}}\label{sec:passivelowerproof}
We break down the proof into several steps.
For now, we assume that the labeled examples are sampled i.i.d.\ as
per the classic PAC setup. At the end, we show how to extend the proof
to the semi-supervised setting.

{\bf (i) Defining a family of adversarial distributions.}
Let $T$ be a $\bart$-net of $\cX$
of size $\Theta(\sqrt m)$
and $\noise=\Theta(1/\sqrt m)$.
For any passive learning algorithm mapping i.i.d.\ samples of size
$\ell=\tilde\Theta(\sqrt m)$
to hypotheses $\hat h_\ell:\cX\to\setpm$,
we construct a random adversarial distribution
$\Ds$
with
agnostic error $\noise$.
We accomplish this via
the construction
described in the proof of Theorem~\ref{thm:passivelower-gen},
with
$|T|=d=\Theta(\sqrt m)$.
The marginal distribution over
$T=\set{x_1,\ldots,x_{d}}$
puts a mass of $1-p
$ on $x_{d}\in T$
and spreads the remaining mass uniformly over the other points,
as in (\ref{eq:D(X)}).
The ``heavy'' point has a deterministic label and the remaining
``light'' points have noisy labels drawn from a random distribution
with symmetric noise
level $b$,
as in (\ref{eq:D(Y|X)}).
We proceed to choose $b$ and $p$;
namely,
\beq
p=\frac{d-1}{2\ell}\sqrt\noise=\tilde\Theta(m^{-1/4}),
\qquad
b=1-\frac{2\noise}{p}=1-\tilde\Theta(m^{-1/4}),
\eeq
which makes
the constraint in (\ref{eq:bpnu}) hold with equality;
this means that the agnostic error is exactly $\noise$
and in particular, establishes (i).

{\bf (ii.a) Lower-bounding the passive learner's error.}
Our choice of $p$ implies that
$\ell p/(d-1)=
\sqrt\noise/2
=:\kappa<1
$.
For this range of $\kappa$,
\citet[Proposition 2.8]{DBLP:journals/corr/KontorovichP16}
implies that
$\hatopt(\kappa,b)=\hf(1-\kappa b)=\Theta(1)$.
Choosing $u=\oo4(1-\kappa b)=\Theta(1)$ in (\ref{eq:main-lb-P}),
Theorem~\ref{thm:passivelower-gen}
implies that
\beq
\inf_{\hl} \sup_{\cD\in\sD(\noise)} \P[\Delta(\hl,\cD)>
\tilde\Omega(m^{-1/4})
]>\Omega(1).
\eeq
In more formal terms, there exist constants
$c_0,c_1>0$ such that
\beqn
\label{eq:iia}
\inf_{\hl} \sup_{\cD\in\sD(\noise)} \P[\Delta(\hl,\cD)>
  c_0
  p
]>c_1.
\eeqn

{\bf (ii.b) Upper-bounding $\nu(\bart)$.}
To establish (ii.b), it suffices to show that
for
$(X_i,Y_i)_{i\in[\ell]}\sim\Ds^\ell$,
we will have $\nu(\bart)={O}(m^{-1/2})$
with sufficiently high probability.
Indeed, the latter condition implies the requisite
upper bound on
$\min_{t>0 : \netsize(t)<m} \genbound(\nu(t),\netsize(t),\delta,m,1)$,
while (i) implies the lower bound, since the latter quantity cannot be
asymptotically smaller than the Bayes error (which coincides with the
agnostic error for $\cH=\setpm^\cX$).

Recall that the $\bart$-net points
$\set{x_1,\ldots,x_{d-1}}$
are the ``light'' ones
(i.e., each has weight $p/(d-1)$)
and
define the random sets $J_j\subset[\ell]$
by
\beq
J_j=\set{i\in[\ell]: X_i=x_j},
\qquad
j\in[d-1]
.
\eeq
In words, $J_j$ consists of the indices $i$ of the sample points for which $X_i$ falls on
the net point $x_j$.
For $y\in\setpm$, put $\tau_j^y=\sum_{i\in J_j}\one[Y_i=y]$
and
define the
{\em minority count} $\xi_j$ at the net point $x_j$ by
\beq
\xi_j \;=\;
\min_{y\in\setpm}\tau_j^y
\;=\;
\hf(|\tau_j^++\tau_j^-|-|\tau_j^+-\tau_j^-|)
.
\eeq
Observe that
by virtue of being a $\bart$-net, $T$ is $\bart$-separated and hence
the only contribution to $\nu(\bart)$ is from the minority counts (to which the ``heavy'' point $x_{d}$
does not contribute due to its deterministic label):
\beq
\nu(\bart) = \oo{\ell}\sum_{j=1}^{d-1}\xi_j.
\eeq
Now
\beq
\E|\tau_j^++\tau_j^-|=
\E|J_j|
= \frac{\ell p}{d-1}
=\Theta(m^{-1/4})
\eeq
and
\beq
\E|\tau_j^+-\tau_j^-|
&=&
\E_{\sigma_j}
{
  \E[|\tau_j^+-\tau_j^-| \big| \sigma_j]
  }
\\
&\ge&
  \E_{\sigma_j}\abs{
  \E[\tau_j^+-\tau_j^- \big| \sigma_j]
  }.
\eeq
Computing
\beq
\E[\tau_j^+ | \sigma_j=+1] = (\oo2+\frac{b}2)\frac{\ell p}{d-1}
,\qquad
\E[\tau_j^- | \sigma_j=+1] = (\oo2-\frac{b}2)\frac{\ell p}{d-1}
,
\eeq
with an analogous calculation when conditioning on $\sigma_j=-1$,
we get
\beq
\E|\tau_j^+-\tau_j^-| \ge
\frac{b\ell p}{d-1}
\eeq
and hence
\beq
\E[\xi_j] &\le& \oo2\paren{
  \frac{\ell p}{d-1}
  -
  b\frac{\ell p}{d-1}
}\\
&=&
(1-b)\frac{\ell p}{2(d-1)}=\frac{2\noise}{p}\cdot\frac{\ell p}{2(d-1)}
=\frac{\noise\ell}{d-1}
.
\eeq
It follows that
\beq
\E[\nu(\bart)] &=& \oo{\ell}\sum_{j=1}^{d-1}\E[\xi_j]\\
&\le& \frac{d-1}{\ell}
\cdot
\frac{\noise\ell}{d-1}
= \noise =\Theta(m^{-1/2}).
\eeq
To give tail bounds on $\nu(\bart)$,
we use Markov's inequality:
for all $c_2>0$,
\beq
\P[\nu(\bart)>c_2\E[\nu(\bart)]\le\oo{c_2}.
\eeq
Choosing $c_2$ sufficiently large that $1-1/c_2>c_1$
(the latter from (\ref{eq:iia}))
implies the existence of constants $c_0,c_2,c_3>0$ such that
\beq
\inf_{\hl} \sup_{\cD\in\sD(\noise)} \P\sqprn{
  \paren{\Delta(\hl,\cD)>
  c_0
  p}
  \wedge
  \paren{
    \nu(\bart)\le c_2\noise
  }
}>c_3.
\eeq
Since $p=\tilde\Theta(m^{-1/4})$
and $\noise=\Theta(m^{-1/2})$,
this establishes (ii) and concludes the proof of \thmref{passivelower}.

{\bf (iii) Extending to the semi-supervised setting.}
Providing the learner with the exact weights of $\cX=[d]$
under our adversarial distribution does not give it any additional power.
Indeed, the information-theoretic excess risk lower bound in
\eqref{KP} hinges on the fact that to estimate $\sigma(x)$
with some desired certainty, the point $x$ must be sampled
some minimal number of times.
The marginal probability of $x$ does not enter that calculation,
and hence knowing its value does not afford the learner an improved performance.

\section{Active Learning Lower Bound}\label{sec:activelower}

We now prove the active learning lower bound stated in \thmref{activelower}. To prove the theorem, we first prove a result which is similar to the classical No-Free-Lunch theorem, except it holds for active learning algorithms. The proof follows closely the proof of the classical No-Free-Lunch theorem given in \citet[Theorem 5.1]{shwartz2014understanding}, with suitable modifications.
\begin{theorem}\label{thm:activenfl}
Let $\beta \in [0,\half)$, and $m$ be an integer.
Let $\cA$ any active learning algorithm over a finite domain $\cX$ which gets as input a random labeled sample $S \sim \cD^m$ (with hidden labels) and outputs $\hat{h}$. 
If $\cA$ queries fewer than $\cX/2$ labels from $S$, then there exists a distribution $\cD$ over $\cX \times \{0,1\}$ such that 
\begin{itemize}
\item Its marginal on $\cX$ is uniform, and for each $x \in \cX$, $\P[Y = 1 \mid X = x]\in\set{\beta,1-\beta}$.
\item $\E[\err(\hat{h},\cD)] \geq \oo4$.
\end{itemize}
\end{theorem}

\begin{proof}
Let $\cF = \{f_1,\ldots,f_T\}$ be the set of possible functions $f_i:\cX \rightarrow \{0,1\}$. Let $\cD_i$ to be a distribution with a uniform marginal over $\cX$, and $\P_{(X,Y) \sim \cD_i}[Y = 1 \mid X = x] = f_i(x)(1-\beta) + (1-f_i(x))\beta$. 
Consider the following random process: First, draw an unlabeled sample $X = (x_1,\ldots,x_m)$ i.i.d.~from $\cD_X^m$. Then, draw $B = (b_1,\ldots,b_m)$ independently from a Bernoulli distribution with $\P[b_i = 1] = \beta$. For $i \in [T]$, let $S^i(X,B) = ((x_1,y_1),\ldots,(x_m,y_m))$ such that $x_i$ are set by $X$, and $y_i = f_i(x)$ if $b_i = 0$ and $1-f_i(x)$ otherwise. Clearly, $S^i(X,B)$ is distributed according to $\cD_i^m$. Let $h_i(S)$ be the output of $\cA$ when the labeled sample is $S$. Denote by $\hat{h}_i$ the (random) output of $\cA$ when the sample is drawn from $\cD_i$. Clearly
\[
\E[\err(\hat{h}_i,\cD_i)] = \E_{X,B}[\err(h_i(S(X,B)),\cD_i)].
\]
Therefore (as in (5.4) in \cite{shwartz2014understanding}), for some $j,X,B$ it holds that 
\begin{equation}\label{eq:jgeq}
\E[\err(\hat{h}_j,\cD_j)] \geq \frac{1}{T}\sum_{i=1}^T\E[\err(\hat{h}_i,\cD_i)] \geq  \frac{1}{T}\sum_{i=1}^T \err(h_i(S(X,B)),\cD_i).
\end{equation}
Fix $X,B,j$ as above, and denote for brevity $h_i := h_i(S(X,B))$. 
Let $V_i$ be the set of examples $x \in \cX$ for which that $\cA$ does not observe their label if the labeled sample is $S^i(X,B)$ (this includes both examples that are not in the sample at all as well as examples that are in the sample but their label is not requested by $\cA$). We have $|V_i| > |\cX|/2$ by assumption. Then 
(as in Eq.~(5.6) therein) 
\begin{equation}\label{eq:sumv}
\frac{1}{T}\sum_{i=1}^T \err(h_i,\cD_i) \geq \frac{1}{T}\sum_{i=1}^T \frac{1}{2|V_i|}\sum_{x \in V_i}\one[h_i(x) \neq f_i(x)].
\end{equation}
Since $\cA$ is active, it selects which examples to request, which can depend on the labels observed by $\cA$ so far. Therefore, $V_i$ can be different for different $i$. However, an argument similar to that of the No-Free-Lunch theorem for the passive case still goes through, as follows. 

Let $i,i'$ such that $f_i(x) = f_{i'}(x)$ for all $x \notin V_i$, and $f_i(x) = 1-f_{i'}(x)$ for all $x \in V_i$. Since $X,B$ are fixed, $\cA$ observes the same labels for all $x \notin V_i$ for both $S^{i'}(X,B)$ and $S^i(X,B)$, thus all its decisions and requests are identical for both samples, and so $V_i = V_{i'}$, and $h_i = h_{i'}$. Therefore, it is possible to partition $T$ into $T/2$ pairs of indices $i,i'$ such that for each 
such pair, 
\begin{align*}
&\frac{1}{2|V_i|}\sum_{x \in V_i}\one[h_i(x) \neq f_i(x)] + \frac{1}{2|V_{i'}|}\sum_{x \in V_{i'}}\one[h_{i'}(x) \neq f_{i'}(x)]\\
&=\frac{1}{2|V_i|}\sum_{x \in V_i}\one[h_i(x) \neq f_i(x)] + \one[h_i(x) \neq 1-f_i(x)]\\
&= \half.
\end{align*}
Therefore, $\frac{1}{T}\sum_{i=1}^T \err(h_i,\cD_i) \geq \oo4$. Therefore, from \eqref{sumv}, $\frac{1}{T}\sum_{i=1}^T \err(h_i,\cD_i) \geq \oo4.$
Combining this with \eqref{jgeq}, it follows that $\E[\err(\hat{h}_j,\cD_j)] \geq \oo4$.
\end{proof}

We will also make use of
the following simple lemma.
\begin{lemma}\label{lem:pplusx}
  Let $\beta \in [0,1]$. Let $\cD$ be a distribution over $\cX \times \{0,1\}$ such that for $(X,Y) \sim \cD$, for any $x$ in the support of $\cD$, $\P[Y = 1 \mid X = x]\in\set{\beta,1-\beta}$.
  Let $N$ be the size of the support of $\cD$. Let $S \sim \cD^m$. Denote by $n_x$ the number of sample pairs $(x',y')$ in $S$ where $x' = x$, and let $n_x^+$ be the number of sample pairs $(x',y')$ where $x' = x$ and $y' = 1$. Let $\hat{p}^+_x = n^+_x/n_x$ (or zero if $n_x = 0$).
Then
\[
2\beta(1-\beta)(m-N)\leq \sum_{x \in \cX}\E[2n_x\hat{p}^+_x(1-\hat{p}^+_x)] \leq  2\beta(1-\beta)m.
\]
\end{lemma}
\begin{proof}
We have
\[
\E[2n_x \hat{p}^+_x(1-\hat{p}^+_x)] = \sum_{i=1}^\infty \P[n_x = i]\cdot i \cdot \E[2\hat{p}^+_x(1-\hat{p}^+_x) \mid n_x = i].
\]
Note that $\E[2\hat{p}^+_x(1-\hat{p}^+_x) \mid n_x = 1] = 0$.
For $i > 1$, let $y_1,\ldots,y_i$ be the labels of the examples that are equal to $x$ in $S$, then
\[
\sum_{j,k \in [i]}\one[y_k \neq y_j] = 2n^+_x(i - n^+_x) = i^2 \cdot 2\hat{p}^+_x(1-\hat{p}^+_x).
\]
Therefore, letting $(X_1,Y_1),(X_2,Y_2) \sim \cD^2$,
\begin{align*}
\E_{S \sim \cD^m}[2\hat{p}^+_x(1-\hat{p}^+_x)\mid n_x = i] &= \frac{1}{i^2}\E_{S \sim \cD^m}[\sum_{j,k \in [n_x]}\one[y_k \neq y_j] \mid n_x = i] \\
&=\frac{i^2-i}{i^2}\P[Y_1 \neq Y_2 \mid X_1 =X_2= x]\\
& = 2(1-\oo i)\beta(1-\beta),
\end{align*}
Thus
\begin{align*}
\E[2n_x\hat{p}^+_x(1-\hat{p}^+_x)] &= 2\beta(1-\beta)\sum_{i=2}^\infty (i-1)\P[n_x = i]\\
&= 2\beta(1-\beta)(\E[n_x]+\P[n_x= 0]-1).
\end{align*}
To complete the proof, sum over all $x$ in the support of $\cD$, and note that $\sum_x \E[n_x] = m$, and $\sum_x (\P[n_x= 0]-1) \in [-N,0]$. 
\end{proof}

We now prove our lower bound, stated in \thmref{activelower},  on the number of queries required by any active learning with
competitive guarantees
similar to ours.

\vspace{1em}
\begin{proof}[of \thmref{activelower}]
Let $N = \floor{\frac{m\alpha - \log(\frac{m}{\delta})}{\log(m)}}$. Let $\beta = 8\alpha \leq \half$. 
Consider a marginal distribution $\cD_X$ over $\cX$ which is uniform over $N$ points $1,\ldots,N \in \reals$. Consider the following family of distributions: $\cD$ such that its marginal over $\cX$ is $\cD_X$, and for each $x \in \cX$, $\P[Y = 1 \mid X = x]\in\set{\beta,1-\beta}$.
Thus the Bayes optimal error for each of these distributions is $\beta$. 

Let $S \sim \cD^m$. If one example in $S$ is changed, $\nu(\half)$ changes by at most $1/m$. Hence, by McDiarmid's inequality \citep{McDiarmid89}, with probability at least $1-\oo {28}$, $\abs{\nu(\half) -  \E[\nu(\half)]} \leq \sqrt{\frac{\log(28)}{2m}}$.
Denote the event that this holds $E_M$. 
Since $\beta = 8\alpha \geq \frac{\log(m) + \log(28)}{\sqrt{2m}}$,
 it follows that under $E_M$,
\begin{equation}\label{eq:numcdiarmid}
\abs{\nu(1/2) -  \E[\nu(1/2)]} \leq \beta/8. 
\end{equation}

We now bound $\E[\nu(\half)]$. Using the notation $\hat{p}^+_x,n_x,n^+_x$ as in \lemref{pplusx}, we have 
\[
\nu(1/2) = \frac{1}{m}\sum_{x \in \cX}\min(n^+_x,n_x-n^+_x) = \frac{1}{m}\sum_{x \in \cX}n_x\min(p^+_x,1-p^+_x)
\]
Also, for all $p \in [0,1]$, $\min(p,1-p) \leq 2p(1-p) \leq 2\min(p,1-p)$.
Therefore
\[
\frac{1}{2m}\sum_{x \in \cX} 2n_x p^+_x(1-p^+_x) \leq \nu(1/2) \leq \frac{1}{m}\sum_{x \in \cX}2 n_x p^+_x(1-p^+_x).
\]
By \lemref{pplusx}, it follows that
\[
\frac{m-N}{m}\beta(1-\beta)\leq \E[\nu(1/2)] \leq  2\beta(1-\beta).
\]
Since $N \leq m/2$ 
and $\beta \in [0,\half]$, $\E[\nu(\half)] \in (\beta/4, 2\beta)$.
Combining this with \eqref{numcdiarmid}, we get that under $E_M$, $\alpha = \beta/8 \leq \nu(\half) \leq \frac{17}{8}\beta = 17\alpha. $

Now, we bound $G_{\min}$ from above and below assuming $E_M$ holds. Denote 
\[
G(t) := \genbound(\nu(t),\netsize(t),\delta, m, 1).
\]
To
establish
a lower bound on $\gmin(m,\delta)$, 
note that $\gmin(m,\delta) = \min_{t > 0}G(t) \geq \min_{t>0}\nu(t)$. For $t \in (0,\half)$, $\nu(t) = \nu(\half)$ (since the distances between any two
distinct
points in $S$ is at least $1$). In addition, since $\nu$ is
monotonically increasing,
we have
$\nu(t) \geq \nu(\half)$
for $t \geq \half$.
Hence  $\min_{t>0}\nu(t) \geq \nu(\half) \geq \beta/8 = \alpha$.

To show an upper bound on $\gmin(m,\delta)$, we upper bound $G(\half)$. 
Note that $\netsize(\half) \leq N$. 
Recall the definition of $\phi(t)$ in \eqref{phidef}. We have 
\[
\phi(\half) = \frac{(N+1)\log(m) + \log(\oo\delta)}{m} \leq \alpha.
\]
Then, since $\nu(\half) \le 17\alpha$,
\[
G(1/2) \leq \frac{m}{m-N}(\nu(\half) + \frac{2}{3}\alpha + \frac{3}{\sqrt{2}}\sqrt{\nu(\half) \alpha}) \leq 30\alpha.
\]
In the last inequality we used the fact that $\frac{m}{m-N} \leq 10/9$. So if $E_M$ holds, $\gmin(m,\delta) \leq G(\half) \leq 30\alpha$.

From the assumption on $\cA$, with probability at least $1-\delta$,
we have
$\err(\hat{h},\cD) \leq C\gmin(m,\delta) \leq 30C\alpha \leq 1/8$ (since $\alpha \leq \frac{1}{240C}$). Let $E_L(\cD)$ denote the event that $\cA$
queries
fewer than $N/2$ labels, where the probability is over the randomness of $S$ and $\cA$. Let $h'$ be the output of an algorithm that behaves like $\cA$ in cases where $E_L(\cD)$ holds, and
queries
at most $N/2$ otherwise. By \thmref{activenfl}, there exists some $\cD$ in the family of distributions such that 
$\E[\err(h',\cD)] \geq \oo4.$ By Markov's inequality, 
$\P[\err(h',\cD) \geq \oo 8] \geq 1/7$. Also, $\P[h' = \hat{h}] \geq \P[E_L(\cD)]$. Therefore 
\[
\P[\err(\hat{h},\cD) \geq 1/ 8] \geq \P[\err(h',\cD) \geq 1/ 8] + \P[E_L(\cD)] - 1 = \P[E_L(\cD)] - 6/7.
\]
Therefore $\P[E_L(\cD)] - 6/7 \leq \P[\err(\hat{h},\cD) \geq \oo 8] \leq \delta$. Since by assumption $\delta \leq 1/14$, it follows that $\P[E_L(\cD)] \leq 6/7 + \delta \leq 13/14$. It follows that with a probability of at least $1/14$, the negation of $E_L(\cD)$ holds. Since also $E_M$ holds with probability at least $1-\frac{1}{28}$, it follows that with a probability of at least $\frac{1}{28}$, both $E_M$ and the negation of $E_L(\cD)$ hold. Now, as shown above, $E_M$ implies the bounds on $\gmin(m,\delta)$ (item 1 in the theorem statement). In addition, the negation of $E_L(\cD)$ implies that $\cA$ queries at least $N/2 =  \half\floor{\frac{m\alpha - \log(\frac {m}{\delta})}{\log(m)}}$ labels (item 2 in the theorem statement). This completes the proof.
\end{proof}

\section{Discussion}\label{sec:discussion}

We have presented an efficient fully empirical proximity-based non parametric active learner. Our approach provides competitive error guarantees for general distributions, in a general metric space, while keeping label complexity significantly lower than any passive learner with the same guarantees. \algname\ yields fully empirical error estimates, easily computable from finite samples. This is in contrast with
classic techniques, that present bounds and rates that depend on unknown distribution-dependent quantities. 

An interesting question is whether the guarantees can be related to the Bayes error of the distribution. Our error guarantees give a constant factor over the error guarantees of \citet{gkn-jmlr17+aistats}. A variant of this approach
\citep{DBLP:journals/tit/GottliebKK14+colt} was shown to be Bayes-consistent \citep{DBLP:conf/aistats/KontorovichW15}, and we conjecture that this holds also for the algorithm of \citet{gkn-jmlr17+aistats}.
The passive component of our learning algorithm is indeed Bayes-consistent \citep{DBLP:conf/nips/KontorovichSW17}.
Since in our analysis
\algname\ achieves a constant factor over the error of
the passive learner, Bayes-consistency
of the active learner
cannot be inferred from our present techniques;
we leave this problem open for future research.

Another important issue is one of efficient implementation. We mentioned that the naive $O(m^2)$ runtime
for constructing a $t$-net may be improved to 
$2^{O(\ddim(\cX))}m\log(1/t)$, as shown in \citet{KL04,GottliebKN14}.
The fast $t$-net construction was the algorithmic work-horse of
\citet{DBLP:journals/tit/GottliebKK14+colt,GottliebKN14,gkn-jmlr17+aistats}
and inspired
the passive component of our learner.
We note that implementing even this passive component
efficiently is far from trivial; this formed the core of a Master's thesis
\citep{Korsunsky17}.
The remaining obstacle to making our algorithm fully practical is the magnitude of some of the
constants. We believe these to be artifacts of the proof and intend to bring them down
to manageable values in future work.

\acks 
Sivan Sabato was partially supported by the Israel Science Foundation (grant No.~555/15). Aryeh Kontorovich was partially supported by the Israel Science Foundation (grants No.~1141/12 and 755/15) and
Yahoo Faculty and Paypal awards.
We thank Lee-Ad Gottlieb and Dana Ron for the helpful discussions, and the referees for
carefully reading the manuscript and their helpful suggestions.

\bibliography{activenn}

\end{document}